\documentclass[10pt,journal,compsoc]{IEEEtran}
\usepackage{epsfig}
\usepackage[utf8]{inputenc}
\usepackage[T1]{fontenc}
\usepackage{booktabs}
\usepackage{amsfonts}
\usepackage{bbm}
\usepackage{nicefrac}
\usepackage{microtype}
\usepackage{graphicx}
\usepackage{algorithmic}
\usepackage{amsmath,amssymb,amsthm,bm}
\usepackage{multirow,makecell,footmisc}
\usepackage[tight,footnotesize]{subfigure}
\usepackage{caption}
\usepackage{color}
\usepackage{url}
\usepackage{cite}
\usepackage[colorlinks=true]{hyperref}
\usepackage{marvosym}
\usepackage{enumerate}
\usepackage{makecell}
\usepackage{comment}
\usepackage{pifont}
\usepackage{afterpage}
\newtheorem{theorem}{Theorem}

\newcommand{\xmark}{\text{\ding{55}}}
\newcommand{\ymark}{\text{\ding{51}}}
\newcommand{\myparagraph}[1]{\vspace{5pt} \noindent \textbf{#1.}}
\begin{document}

\title{From Big to Small:\\ Adaptive Learning to Partial-Set Domains}

\author{Zhangjie~Cao,
	Kaichao~You,
	Ziyang~Zhang,
	Jianmin~Wang,
	and~Mingsheng~Long (\Letter)
	\IEEEcompsocitemizethanks{
		\IEEEcompsocthanksitem Z. Cao, K. You, J. Wang, and M. Long are with the School of Software, BNRist, Tsinghua University. E-mail: \{caozhangjie14,youkaichao\}@gmail.com, \{jimwang,mingsheng\}@tsinghua.edu.cn.
		\IEEEcompsocthanksitem Z. Zhang is with the Data Center Technology Lab at Huawei. E-mail: zhangziyang11@huawei.com.
		\IEEEcompsocthanksitem Z. Cao and K. You contributed equally to this work.
		\IEEEcompsocthanksitem Corresponding author: Mingsheng Long, mingsheng@tsinghua.edu.cn.
		\IEEEcompsocthanksitem A preliminary conference version of this paper appeared in \cite{cite:CVPR18SAN}.
	}
	\thanks{Manuscript received March 2020, revised July 2021.}
}

\markboth{IEEE Transactions on Pattern Analysis and Machine Intelligence}{Cao \MakeLowercase{\textit{et al.}}}

\IEEEcompsoctitleabstractindextext{%
	\begin{abstract}
		Domain adaptation targets at knowledge acquisition and dissemination from a labeled source domain to an unlabeled target domain under distribution shift. Still, the common requirement of \emph{identical class space} shared across domains hinders applications of domain adaptation to partial-set domains. Recent advances show that deep pre-trained models of large scale endow rich knowledge to tackle diverse downstream tasks of small scale. Thus, there is a strong incentive to adapt models from large-scale domains to small-scale domains. This paper introduces Partial Domain Adaptation (PDA), a learning paradigm that relaxes the identical class space assumption to that the source class space subsumes the target class space. First, we present a theoretical analysis of partial domain adaptation, which uncovers the importance of estimating the \emph{transferable probability} of each class and each instance across domains. Then, we propose Selective Adversarial Network (SAN and SAN++) with a \emph{bi-level selection} strategy and an adversarial adaptation mechanism. The bi-level selection strategy up-weighs each class and each instance simultaneously for source supervised training, target self-training, and source-target adversarial adaptation through the transferable probability estimated alternately by the model. Experiments on standard partial-set datasets and more challenging tasks with superclasses show that SAN++ outperforms several domain adaptation methods.
	\end{abstract}

	}

\maketitle

\IEEEpeerreviewmaketitle
\IEEEdisplaynotcompsoctitleabstractindextext

\section{Introduction}

\IEEEPARstart{D}{eep} learning has made tremendous progress in solving cognitive tasks and data-intensive tasks~\cite{lecun_deep_2015}, while the impressive performance gains rely on large-scale annotated data~\cite{goodfellow_deep_2016}. Collecting sufficient data for diverse applications is costly and unrealistic. Domain adaptation reduces the annotation cost of the domain of interest (the target domain) by adapting knowledge learned from labeled data available in a related domain (the source domain)~\cite{pan2009survey,tan2018survey}.
The main technical challenge of domain adaptation is the distribution shift between the source and target domains~\cite{quionero-candela_dataset_2009,torralba_unbiased_2011}, which violates the \emph{independent and identically distributed} (i.i.d.) assumption of standard statistical learning~\cite{bishop_pattern_2006,vapnik_nature_2013}. Learning domain-invariant representations is a feasible solution to the distribution shift problem, through which the source classifier can be applied to the target domain~\cite{long_learning_2015,cite:ICML15DANN}.

Existing domain adaptation methods usually assume that the source and target domains share \textit{an identical class space}~\cite{long2015learning,cite:JMLR16DANN,NEURIPS2018_ab88b157}. The assumption is easily violated in practical applications since it is unrealistic to verify the identical class space assumption with unlabeled target data. With the emergence of deep models trained on large-scale datasets, transferring knowledge from large datasets to solve downstream tasks has become more realistic and useful. Motivated by this, we relax the identical class space assumption to that the source class space subsumes the target class space. For example, the source domain can be a large-scale dataset, \textit{e.g.}, ImageNet-1K~\cite{cite:ILSVRC15} and OpenImage~\cite{cite:OpenImage} with comprehensive classes, while the target domain can be a small-scale dataset with specific classes. Domain adaptation under this relaxed assumption is coined \textit{Partial Domain Adaptation} ({PDA}).

Partial domain adaptation facilitates a more practical ``\emph{from big to small}'' scenario, \emph{\emph{i.e.}}, transferring knowledge from big domains with broad classes to small domains with narrow classes. This is in parallel with the remarkable success in deep learning, that deep pre-trained models of large-scale endow rich knowledge to tackle diverse downstream tasks of small-scale \cite{Brown20a,chen2020simple}. For example, a self-driving company may collect a large-scale traffic sign dataset in many countries with rich diversity, and transfer the learned model to a new country where traffic signs are not very diverse but it is cumbersome to collect and annotate data~\cite{arcos-garcia_deep_2018,stallkamp_german_2011}.

\begin{figure*}[tbp]
	\centering
	\subfigure[Challenge in PDA]{
		\includegraphics[width=0.45\textwidth]{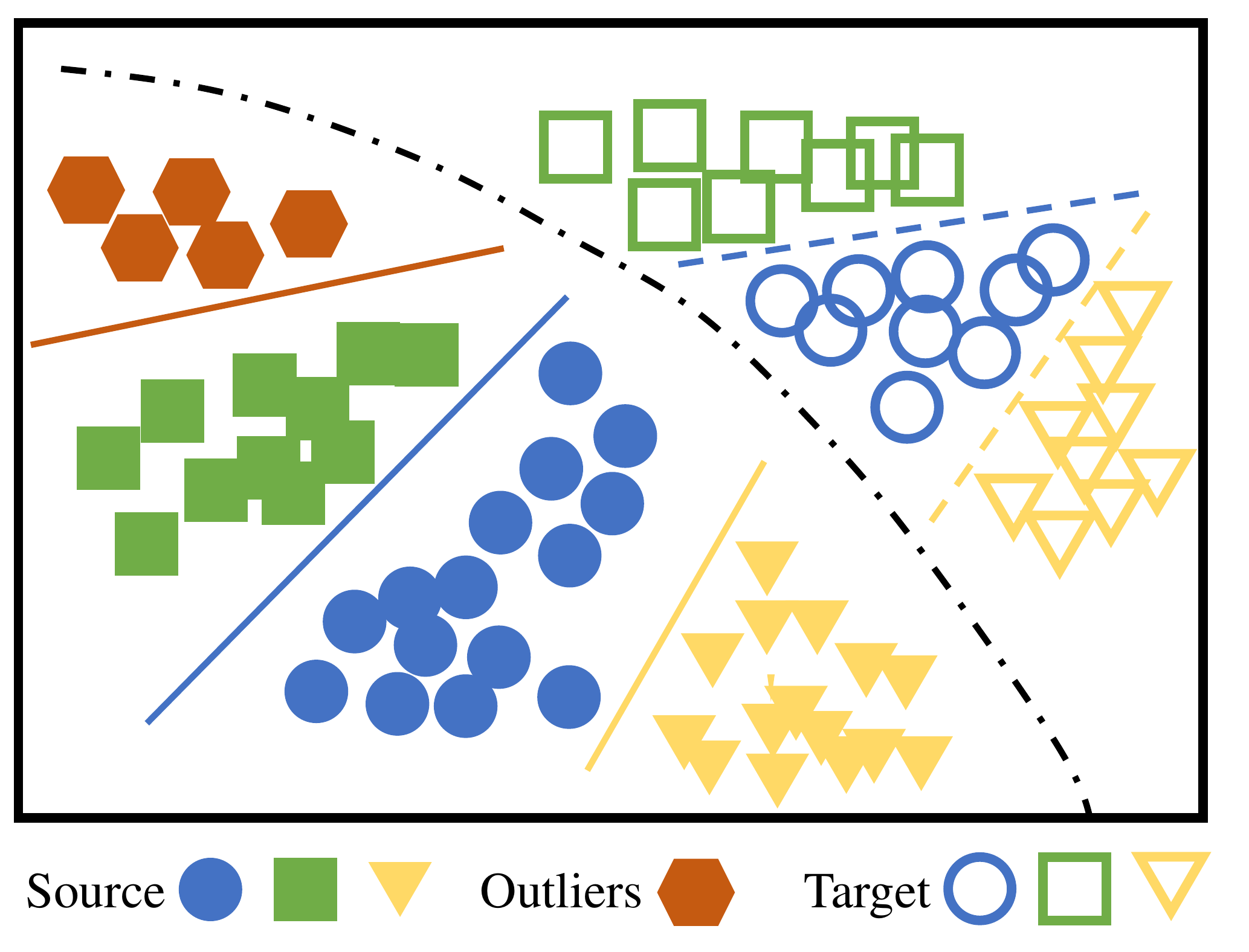}
		\label{fig:san_problem}
	}\hfil
	\subfigure[Solution to PDA]{
		\includegraphics[width=0.45\textwidth]{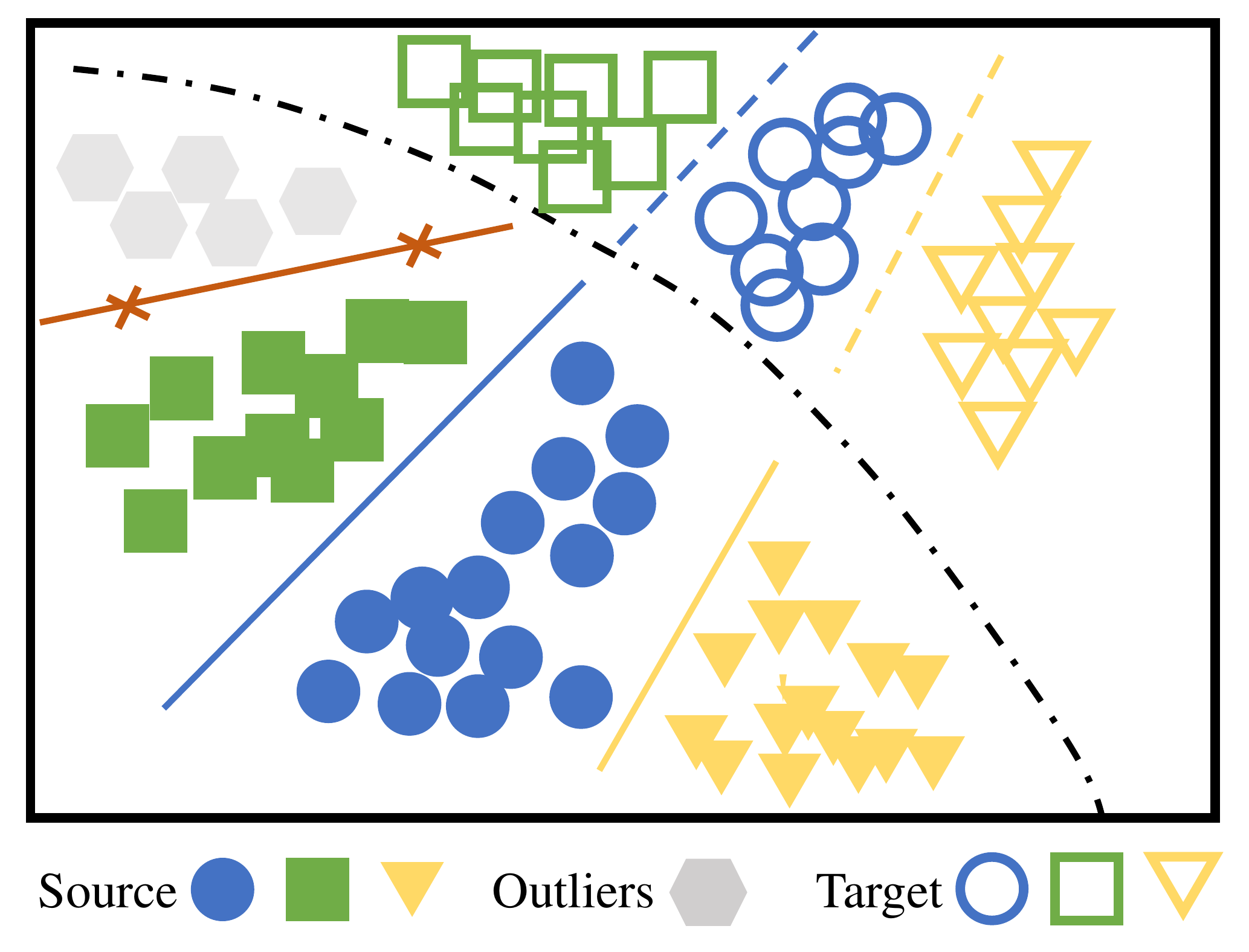}
		\label{fig:san_solution}
	}
	\caption{The challenge and solution of partial domain adaptation (PDA). The solid shapes show the data in the source domain and the hollow shapes show the data in the target domain. Each shape represents a class. The main assumption is that the source class space subsumes the target class space. (a) The challenge is caused by the \emph{source-specific outlier class} (hexagon) that is absent from the target domain, as well as the \emph{source-target distribution shift} on the shared classes (circle, square, and triangle) shown by the different decision boundaries (solid lines) for the source and target domains. (b) We propose a \emph{bi-level selection} strategy to eliminate the influence of source-specific outlier examples so as to promote discriminative learning and distribution alignment in the shared class space.}
	\label{fig:SANproblem}
\end{figure*}

As illustrated in Fig.~\ref{fig:SANproblem}, partial domain adaptation (PDA) faces the \emph{negative-positive transfer dilemma}. (1) Forcing knowledge transfer from the whole source domain to the target domain is at the risk of \emph{negative transfer} on the target domain, since the extra source knowledge is not relevant to the target task. To avoid negative transfer, we should avoid transferring any source-specific class knowledge to the target domain. For example, in Fig.~\ref{fig:SANproblem}, we should eliminate the source class `hexagon' during knowledge transfer. (2) To promote \emph{positive transfer}, we not only need to uncover the domain-shared class space, but also need to further explore the class-conditional information to enhance the alignment of multimodal distributions in the shared class space. For example, in Fig.~\ref{fig:SANproblem}, we need to transfer knowledge from the source classes `circle', `square' and `triangle' (in solid shapes) to corresponding target classes `circle', `square' and `triangle' (in hollow shapes).

The key idea to address partial domain adaptation (PDA) is to recognize which part of the source class space can be shared with the target class space. Intuitively, the contribution of each source class to the target domain can be characterized by the \emph{transferable probability} over all source classes, \emph{i.e.}, source classes with higher transferable probability are more likely to be shared with the target class space. In this paper, we theoretically analyze the PDA problem and provide a bound on the estimation error of the transferable probability for each source class, which sheds a theoretical light on the algorithm design.

Based on the theoretical analysis, we propose Selective Adversarial Network (SAN) to address the PDA problem. We develop a \emph{bi-level selection} strategy consisting of class selection and instance selection. In \emph{class selection}, we estimate the transferable probability of each source class from source labeled data and target unlabeled data, which ensures that source supervised training, target self-training and source-target adversarial adaptation are confined within the shared class space. This avoids \emph{negative transfer} caused by the source-specific outlier classes when transferring ``from big to small''. In \emph{instance selection}, we utilize a multi-task discriminator with each task enabling distribution alignment conditioned on a source class and estimating the transferable probability of each target instance to ensure that target instances are assigned to the correct alignment task. This promotes \emph{positive transfer} by aligning the class-conditioned feature distributions in the shared class space. We perform extensive empirical studies on several benchmarks for both standard domain adaptation and partial domain adaptation settings. The results show that the proposed SAN approach outperforms previous works significantly.

This journal paper extends our conference paper~\cite{cite:CVPR18SAN} in a number of important perspectives:
\begin{itemize}
	\item We provide a \emph{theoretical analysis} on the PDA problem and unveil the importance of estimating the transferable probability of each class and each instance across domains in the journal version. This provides a theoretical foundation for our algorithm design of the \emph{bi-level selection} strategy in the adversarial adaptation framework.

	\item For the method design, our conference version, coined as SAN, proposes multiple discriminators and applies the instance selection and class selection to the discriminators, which confines \emph{only the adaptation procedure} into the shared class space. The journal version, coined as SAN++, further proposes a self-training loss to enhance the target prediction and enables class selection on all training losses including source supervised training, target self-training and source-target adversarial adaptation. This strategy confines \emph{the whole learning procedure} into the shared class space which further promotes positive transfer and aviods negative transfer.

	\item The proposed SAN and SAN++ are extensively tested on common datasets for partial domain adaptation. SAN++ achieves state-of-the-art performances on all these datasets. Furthermore, when partial domain adaptation becomes more challenging with target classes being subclasses of source classes, many PDA methods fail but SAN++ still has a decent performance, demonstrating the generality of SAN++.
\end{itemize}

\section{Related Work}

In this section, we review the relevant literature on domain adaptation with and without labeled data in the target domain. When target labeled data are available, the problem is typically solved by pre-training the model with the source labeled data and then fine-tuning the model on the target labeled data, which is called ``transfer learning''. Prior works on this transfer learning paradigm either improve the quality of the pre-trained model such as pre-training on large-scale datasets~\cite{cite:ICML14DeCAF,he2019rethinking} or self-supervised learning~\cite{dosovitskiy2015discriminative,chen2020simple}, or develop advanced algorithms to enable fine-tuning~\cite{kou_stochastic_2020,you_co-tuning_2020}. Transfer learning with few labeled data is addressed by semi-supervised methods~\cite{cite:ICCV15SDT,herath_learning_2017,cite:CVPR17DAMA,cite:ICCV17SDA,wang_self-tuning_2021}. In this paper, we focus on transfer learning without any target labeled data, \emph{i.e.} domain adaptation, which has been studied by various methods prior to deep learning~\cite{cite:ECCV10Office,cite:TNN11TCA,cite:ICML13TCS,cite:NIPS14FTL}. However, their empirical success is now surpassed by deep domain adaptation methods. We refer readers to Pan \& Yang \cite{pan2009survey} for a survey of methods before the era of deep learning.

We focus on the class space relationship between the source domain and the target domain in this paper. The specific applications of domain adaptation such as detection~\cite{rodriguez_domain_2019}, de-stylization~\cite{shiri_identity-preserving_2019} and domain selection~\cite{zhang_model_2018} are not the focus of the paper. Table~\ref{tab:related_work} shows different domain adaptation settings based on the relationship between the source and target class spaces. We review the methods of each setting below.

\begin{table}[tbp]
	\centering
	\addtolength{\tabcolsep}{6pt}
	\caption{Comparison of different domain adaptation settings based on their assumptions. $\mathcal{C}_s$ and $\mathcal{C}_t$ denote the source class space and target class space respectively.}
	\label{tab:related_work}
	\begin{tabular}{l|c}
		\toprule
		Problem Setting                                                       & Assumption                                     \\
		\midrule
		Closed-Set Domain Adaptation (DA)~\cite{long_learning_2015}           & $\mathcal{C}_t=\mathcal{C}_s$                  \\
		Open-Set Domain Adaptation (OSDA)~\cite{panareda_busto_open_2017}     & $\mathcal{C}_t\supseteq\mathcal{C}_s$          \\
		Partial Domain Adaptation (PDA)~\cite{cite:CVPR18SAN,cite:CVPR18IWAN} & $\mathcal{C}_t\subseteq \mathcal{C}_s$         \\
		Universal Domain Adaptation (UniDA)~\cite{you_universal_2019}         & $\mathcal{C}_t\cap\mathcal{C}_s \ne \emptyset$ \\
		\bottomrule
	\end{tabular}
\end{table}

\subsection{Closed-Set Domain Adaptation}
\label{sec:close_da}

Closed-set domain adaptation addresses the fundamental technical challenge of distribution shift under the assumption that source and target domains share an identical class space.

Minimizing a well-defined \emph{statistical discrepancy} between feature distributions across domains is the mainstream unsupervised domain adaptation approach. Deep adaptation network (DAN)~\cite{long_learning_2015, long_transferable_2019} learns transferable features across domains by maximizing the test power of the statistical distance, \emph{i.e.}, multi-kernel maximum mean discrepancy (MK-MMD), and minimizing the feature distance to generate domain-invariant representations.
Residual transfer network (RTN)~\cite{cite:NIPS16RTN} further introduces a shortcut path and adopts entropy minimization criterion to utilize the inductive bias of the cluster assumption. Joint adaptation network (JAN)~\cite{cite:ICML17JAN} closes the domain gap of the joint distribution of features and predictions. CORAL~\cite{cite:AAAI16CORAL} aligns the covariance of features between domains. CMD~\cite{cite:ICLR17CMD} defines Central Moment Discrepancy to quantify the discrepancy across domains via multiple orders of moments. Associative Domain Adaptation~\cite{cite:ICCV17ADA} constructs a bipartite graph to estimate the associative similarity between each data pair, which avoids mismatching features across classes. Maximum Classifier Discrepancy (MCD)~\cite{cite:CVPR18MCD} maximizes the disagreement of two target classifiers and optimizes the feature extractor to minimize the disagreement in a min-max optimization framework. The min-max game is also used by Herath~\textit{et al.} \cite{herath_min-max_2019} to obtain state-of-the-art performance in both domain adaptation and zero-shot learning.

Matching distributions across domains by \emph{adversarial learning} is another cornerstone approach to unsupervised domain adaptation. The min-max framework in Generative Adversarial Networks (GAN)~\cite{cite:NIPS14GAN} adopts a discriminator and a generator with opposite goals, which finally reaches a dynamic equilibrium. GAN and its fancy variants~\cite{radford_unsupervised_2015, karras_progressive_2018} have shown significant success in generating realistic images. Such adversarial learning paradigm has been applied to domain adaptation to align the source and target distributions, where a domain discriminator distinguishes between features from the source domain and the target domain, and a feature extractor aims to confuse the domain discriminator in an adversarial training paradigm~\cite{cite:ICML15DANN,cite:JMLR16DANN,cite:CVPR17ADDA}.
LEL~\cite{cite:NIPS17LEL} enhances adversarial discriminative domain adaptation by restricting the representation of each target example to be similar to a few source data points using entropy minimization. Volpi~\textit{et al.} \cite{cite:CVPR18AFA} proposed feature augmentation to use GAN to directly generate domain-invariant features and show promising results in cross-modal object recognition.

Another branch of closed-set domain adaptation is \emph{pixel-level domain adaptation}, which generates labeled images similar to the target domain.
There are various approaches to generating images from deep features. Hu and Sankaranarayanan~\textit{et al.} \cite{cite:CVPR18GenerateAdapt,cite:CVPR18DupGAN} generate source and target domain images from features of intermediate network layers with GAN, and employ an AC-GAN~\cite{odena_conditional_2017} domain discriminator for adaptation. Their auxiliary domain discriminator can both preserve the class information of images and adapt features across domains. DupGAN~\cite{cite:CVPR18DupGAN} provides an additional domain code to the generator to guide the source-target and target-source translation process. GTA~\cite{cite:CVPR18GenerateAdapt} transforms source features to source-like images and requires target features to be close to source features. CDRD~\cite{liu_detach_2018} focuses on disentangling cross-domain features. A recent line of research focuses on generating target images from images rather than features. PixelDA~\cite{cite:CVPR17PLDA} generates images conditioned on source domain images but requires the images to be similar to the target domain. Inspired by CycleGAN~\cite{cite:ICCV17CycleGAN}, which shows impressive results in unpaired image-to-image translation, cycle-consistency-based domain adaptation methods~\cite{cite:ICML18CYCADA,cite:CVPR18BDAGAN} have been developed. The cycle-consistency loss requires that the generated target images can be transformed back to the original source images. These methods further enforce class consistency to preserve class information in the source-target-source transformation loop. Murez~\textit{et al.} \cite{cite:CVPR18IIT} proposed a loss function to unify classification loss, domain adversarial loss, reconstruction loss, translation adversarial loss, cycle consistent loss and class preserving loss.

Both feature-level and pixel-level domain adaptation methods require the same class space between the source domain and the target domain, which are not applicable in the partial domain adaptation setting studied in this paper.

\subsection{Open-Set Domain Adaptation}
\label{sec:openset}

Open-set domain adaptation generalizes open-set recognition~\cite{scheirer_toward_2013} to the domain adaptation setup. It relaxes the identical class space assumption of closed-set domain adaptation to that there are test data belonging to none of the source classes. The key to open-set domain adaptation is to eliminate the open classes in the target class space. {Busto and Gall}~\cite{panareda_busto_open_2017} proposed a cluster-based ATI algorithm for open-set domain adaptation. ATI requires that the shared classes between domains are known, therefore, limiting its practical significance. {Saito~\textit{et al.}} \cite{saito_adversarial_2018} studied a variant of open-set domain adaptation, where the target class space subsumes source categories. The classifier is extended with an open class which is trained adversarially. The adversarial training requires a prior of the open classes. {Liu~\textit{et al.}} \cite{liu_separate_2019} improved the solution with a ranking mechanism in each mini-batch to cancel the necessity of a manual threshold. {You~\textit{et al.}} \cite{you_universal_2019} further studied the case where no assumption about the class space relationship is imposed, which is dubbed universal domain adaptation. They combined both class and domain information to distinguish the open classes and promote the adaptation between the shared class space. {Baktashmotlagh~\textit{et al.}} \cite{baktashmotlagh_learning_2019} exploited subspace learning to discover the shared structure of both domains so that target private data can be detected. {Fang~\textit{et al.}} \cite{fang_open_2020} derived a generalization bound for open-set domain adaptation.

Contrary to our setting, open-set domain adaptation allows novel classes in the target class space and corresponds to a ``from small to big'' problem. Research efforts in open-set domain adaptation lay a foundation for relaxing the identical class space assumption in domain adaptation.

\subsection{Partial Domain Adaptation}

Closed-set domain adaptation methods provide fundamental mathematical tools for learning from different domains, but the assumption of an identical class space cannot be verified due to the absence of target labels, which restricts their applications. {Tas and Koniusz}~\cite{tas_cnn-based_2018} studied domain adaptation across two domains with different class spaces when a few target labels are available. The problem with an unlabeled target domain is significantly more difficult. Motivated by the release of large-scale datasets (such as ImageNet-1K~\cite{cite:ILSVRC15}), partial domain adaptation is introduced  in our conference version \cite{cite:CVPR18SAN}, by relaxing the identical class space assumption to that the target class space is a subset of the source class space. Most partial domain adaptation methods design weighting schemes to re-weight source examples during domain alignment. IWAN~\cite{cite:CVPR18IWAN} develops a domain discriminator to produce the probability of an input belonging to the shared class space and uses the probability to re-weight data in adversarial domain alignment. PADA~\cite{cao_partial_2018} further applies the class-level weight on the source classifier. ETN~\cite{cao_learning_2019} leverages an auxiliary classifier to achieve state-of-the-art performance. DRCN~\cite{li2020deep} replaces the adversarial alignment with MMD-based alignment and uses a residual corrector~\cite{cite:NIPS16RTN} to enable asymmetric feature mapping for both domains. {BA$^3$US~\cite{liang2020baus} combines balanced adversarial alignment with adaptive uncertainty suppression, thereby providing an alternative weighting scheme.}

Besides the weighting scheme, there emerge some hard selection methods to tackle partial domain adaptation: once an example or a class is treated as an outlier, it is discarded and never used again. {Fariba~\textit{et al.}} \cite{zohrizadeh_class_2019} select a subset of classes in the source domain by $L_{1, \inf}$ norm penalization of the weight matrix. {Chen~\textit{et al.}} \cite{chen_selective_2020} incorporate reinforcement learning~\cite{sutton_reinforcement_2018} to select useful examples by rewarding data with low reconstruction error. {Hu \textit{et al.}~\cite{hu_discriminative_2020} try to separate source positive distribution and source negative distribution besides hard selection.} {Chen~\textit{et al.}} \cite{chen_domain_2020} incorporate deep Q-learning~\cite{mnih_human-level_2015} into domain adversarial learning and design a reward function depending on the relevance between the selected source instances and target data. Hard selection can be useful but may suffer from high variance. Once a wrong decision is made and a useful source example is discarded, it is impossible to bring it back to the training.

\section{Partial Domain Adaptation}

Given a labeled source domain $\mathcal{D}_s= \{(\mathbf{x},y)\}$ with class space $\mathcal{C}_s$ and an unlabeled target domain $\mathcal{D}_t= \{(\mathbf{x})\}$ with class space $\mathcal{C}_t$, partial domain adaptation aims to improve the performance of the task in the target domain by leveraging the labeled data in the source domain, where the source and target class spaces satisfying $\mathcal{C}_t \subseteq \mathcal{C}_s$.

\myparagraph{Notations} The \emph{shared} class space $\mathcal{C}$ between the source and target domains is defined as $\mathcal{C}= \{y| y \in \mathcal{C}_s, y \in \mathcal{C}_t\} = \mathcal{C}_s \cap \mathcal{C}_t$, \emph{i.e.}, the intersection of source class space and target class space. Since we assume $\mathcal{C}_t \subseteq \mathcal{C}_s$, then $\mathcal{C} = \mathcal{C}_t$. For simplicity, we denote the \emph{outlier} source class space that is irrelevant to the target class space as $\bar{\mathcal{C}}_s=\mathcal{C}_s \backslash \mathcal{C}$.
Then the source subdomain in the shared class space $\mathcal{C}$ is
\begin{equation*}
	\mathcal{D}_s^{\mathcal{C}} \triangleq \{(\mathbf{x},y) \in \mathcal{D}_s  | y \in \mathcal{C}\}.
\end{equation*}
Under such notations, $\mathcal{D}_s^{\bar{\mathcal{C}}_s}$ will represent the source subdomain in the outlier source class space. \textit{It is important to note that the target class space $\mathcal{C}_t$ is unknown and thus $\mathcal{C}$ is unknown throughout training.} It is left for the algorithm developers to discover the shared class space as well as the subdomains in partial domain adaptation.

\myparagraph{Challenges}
From the viewpoint of statistical learning, the source and target domains are sampled from different distributions $p$ and $q$, $p\ne q$. PDA is to build a deep model that learns a transferable feature extractor $\mathbf{f} = F\left( {\bf{x}} \right)$ and an adaptive classifier $\hat{y} = G\left( {\bf{f}} \right)$ to reduce the distribution shift, such that the risk ${\mathbb{E} _{\left( {{\mathbf{x}},y} \right) \sim q}}\left[ {G \left( F({\mathbf{x}}) \right) \ne y} \right]$ is minimized by leveraging the source domain supervision.

In domain adaptation, a main challenge is that the source classifier cannot be directly applied to the target domain due to the distribution shift. In partial domain adaptation (PDA), with different class spaces between the source and target domains, a new challenge is to discover which part of the source class space is sharable with the unknown target class space $\mathcal{C}$. The challenges lead to two technical tasks. On the one hand, matching the whole distributions $p$ and $q$ as in standard domain adaptation will cause \emph{negative transfer} since a part of target class space $\mathcal{C}_t$ will be forcefully matched to the source outlier class space $\bar{\mathcal{C}}_s$. Thus, how to decrease the influence of the outlier source labeled data is crucial to partial domain adaptation. On the other hand, there is still a distribution shift between source and target domains in the shared class space, \emph{i.e.}, $p_{\mathcal{C}} \ne q$, where $p_{\mathcal{C}}$ denotes the distribution of the source domain associated with the shared class space $\mathcal{C}$. Reducing the distribution discrepancy between $p_{\mathcal{C}}$ and $q$ is still important to promoting \emph{positive transfer} of the source knowledge to the target domain.

\section{Theoretical Analysis}
\label{sec:theory}

Intuitively, if $\mathcal{C}$ is known, the PDA problem can be readily converted to a standard domain adaptation problem with $\mathcal{D}_s^{\mathcal{C}} = \{(\mathbf{x},y) \in \mathcal{D}_s  | y \in \mathcal{C}\}$ as the new source domain. We first define the $|\mathcal{C}_s|$-dimensional one-hot label $\dot{\mathbf{y}}$ for each target example. Then the \emph{transferable probability} over source classes for the target data can be estimated by $\dot{\mathbf{w}} = {\mathbb{E}_{({\mathbf{x}},y) \sim q}}\dot{\mathbf{y}} \in \mathbb{R}^{|\mathcal{C}_s|}$, which well reflects the information of the shared class space: $\forall c \in \mathcal{C}_s, \Pr(c \in \mathcal{C}) = \dot{\mathbf{w}}_c$.
However, due to the unsupervised nature of the PDA problem, the target domain label $\dot{\mathbf{y}}$ is unknown. Thus, we instead estimate the transferable probability of source classes by the prediction $\hat{\mathbf{y}}$ of target data, which is $\mathbf{w} = {\mathbb{E}_{({\mathbf{x}},y) \sim q}}\hat{\mathbf{y}}.$ A natural question now is how to bound the estimation error of $\mathbf{w}$, which is answered in Theorem~\ref{thm:bound}.

\begin{theorem}
	\label{thm:bound}
	Let $\Delta^{|\mathcal{C}_s|}$ be a $|\mathcal{C}_s|$-dimensional probability simplex $\Delta^{|\mathcal{C}_s|} = \{\mathbf{y} \in \mathbb{R}^{|\mathcal{C}_s|} | 0 \le y_i \le 1 (1 \le i \le |\mathcal{C}_s|) \land \sum_{j=1}^{|\mathcal{C}_s|} y_i= 1 \}$, $h : \mathbf{x} \mapsto \Delta^{|\mathcal{C}_s|}$ be a classifier that aims to classify source and target data with the correct label, $\phi(\hat{\mathbf{y}}) = \arg \max \hat{\mathbf{y}}$ be the $\arg \max$ operator over the prediction $\hat{\mathbf{y}} = h(\mathbf{x})$, $h_{\mathcal{C}}$ be the classifier confined within $\mathcal{C}$, $\mathcal{E}_p(h)= \mathbb{E}_{({\mathbf{x}},y)\sim p} \; \mathbf{1}(\phi(h(\mathbf{x})) \neq y)$ be the expected error of classifier $h$ with respect to the distribution $p$, $\mathcal{E}_{I}(h)= \mathbb{E}_{({\mathbf{x}},y)\sim q} \; \mathbf{1}(\phi(\hat{\mathbf{y}}) \notin \mathcal{C})$ be the expected error of classifying a target data point as a category in outlier source class space, $\delta (\hat{\mathbf{y}}) = 1 - \max \hat{\mathbf{y}}$ be the compliment of the classifier confidence value, and $\bar{\delta} = \mathbb{E}_{({\mathbf{x}},y)\sim q} \delta (\hat{\mathbf{y}})$ be the expected value of $\delta$ over the target distribution $q$, then the approximation error of the transferable probability in $L_1$ norm can be bounded as follows,
	\begin{equation}\label{eqn:bound}
		\Vert \dot{\mathbf{w}} - {\mathbf{w}} \Vert_1
		\le 2\bar{\delta} + 2 \mathcal{E}_{I}(h) + 2 \mathcal{E}_{p_{\mathcal{C}}}(h_{\mathcal{C}}) + 2 d_{\mathcal{H}\Delta\mathcal{H}}(p_{\mathcal{C}}, q).
	\end{equation}
\end{theorem}

\begin{proof}
	The inequality can be derived as follows,
	\begin{equation*}
		\begin{aligned}
			\Vert \dot{\mathbf{w}} - {\mathbf{w}} \Vert_1 & = \Vert {\mathbb{E}_{({\mathbf{x}},y)\sim q}} \dot{\mathbf{y}}  - {\mathbb{E}_{({\mathbf{x}},y) \sim q}} \hat{\mathbf{y}} \Vert_1                               \\
			                                              & = \Vert {\mathbb{E}_{({\mathbf{x}},y)\sim q}} \; (\dot{\mathbf{y}}  -  \hat{\mathbf{y}} ) \Vert_1                                                               \\
			\text{use} \; \text{a)}                       & \le \mathbb{E}_{({\mathbf{x}},y)\sim q} \Vert \dot{\mathbf{y}} - \hat{\mathbf{y}} \Vert_1                                                                       \\
			\text{use} \; \text{b)}                       & = \mathbb{E}_{({\mathbf{x}},y)\sim q} \; \mathbf{1}(\phi(\dot{\mathbf{y}}) = \phi(\hat{\mathbf{y}})) \Vert \dot{\mathbf{y}} - \hat{\mathbf{y}} \Vert_1          \\
			                                              & \quad + \mathbb{E}_{({\mathbf{x}},y)\sim q} \; \mathbf{1}(\phi(\dot{\mathbf{y}}) \neq \phi(\hat{\mathbf{y}})) \Vert \dot{\mathbf{y}} - \hat{\mathbf{y}} \Vert_1 \\
			\text{use} \; \text{c)}                       & \le \mathbb{E}_{({\mathbf{x}},y)\sim q} \; 2 \delta (\hat{\mathbf{y}}) \mathbf{1}(\phi(\dot{\mathbf{y}}) = \phi(\hat{\mathbf{y}}))                              \\
			                                              & \quad + \mathbb{E}_{({\mathbf{x}},y)\sim q} \; 2 \mathbf{1}(\phi(\dot{\mathbf{y}}) \neq \phi(\hat{\mathbf{y}}))                                                 \\
			\text{use} \; \text{d)}                       & \le 2\bar{\delta} + 2 \mathbb{E}_{({\mathbf{x}},y)\sim q} \; \mathbf{1}(\phi(\dot{\mathbf{y}}) \neq \phi(\hat{\mathbf{y}}))                                     \\
			\text{use} \; \text{b)}                       & = 2\bar{\delta} + 2 \mathbb{E}_{({\mathbf{x}},y)\sim q} \; \mathbf{1}(\phi(\hat{\mathbf{y}}) \notin \mathcal{C})                                                \\
			                                              & \quad + 2 \mathbb{E}_{({\mathbf{x}},y)\sim q} \; \mathbf{1}(\phi(\hat{\mathbf{y}}) \in \mathcal{C},  \phi(\dot{\mathbf{y}}) \neq \phi(\hat{\mathbf{y}}))        \\
			                                              & = 2\bar{\delta} + 2 \mathcal{E}_{I}(h)  + 2 \mathcal{E}_q(h_{\mathcal{C}})                                                                                      \\
			\text{use} \; \text{e)}                       & \le 2\bar{\delta} + 2 \mathcal{E}_{I}(h) + 2 \mathcal{E}_{p_{\mathcal{C}}}(h_{\mathcal{C}}) + 2 d_{\mathcal{H}\Delta\mathcal{H}}(p_{\mathcal{C}}, q)
		\end{aligned} \\
	\end{equation*}
	where a/b/c/d/e refers to the following facts:
	\begin{enumerate}[a)]
		\item The expectation inequality $\Vert \mathbb{E} x \Vert_1 \le \mathbb{E} \Vert x \Vert_1 $.
		\item The identity of the indicator function $\mathbb{E} x =  \mathbb{E} \mathbf{1}(b)x + \mathbb{E} \mathbf{1}(\lnot b)x $, where $b$ is a boolean expression and $\lnot b$ is its negation. The identity is frequently used in the analysis of martingale~\cite{doob_regularity_1940} (a category of stochastic process rooted in fair games).
		\item The property of the probability simplex: for any one-hot label $\dot{\mathbf{y}}$ and a prediction $\hat{\mathbf{y}} \in \Delta^{|\mathcal{C}_s|}$, their difference is $ \Vert \dot{\mathbf{y}} - \hat{\mathbf{y}} \Vert_1  =
			      \begin{cases}
				      2 \delta (\hat{\mathbf{y}}),                           & \phi(\dot{\mathbf{y}}) = \phi(\hat{\mathbf{y}})    \\
				      2 - 2 \hat{\mathbf{y}}_{\phi(\dot{\mathbf{y}})} \le 2, & \phi(\dot{\mathbf{y}}) \neq \phi(\hat{\mathbf{y}}) \\
			      \end{cases}$ \\
		\item For a positive random variable $x$ and a boolean expression $b$ defined over $x$, $\mathbb{E} \mathbf{1}(b)x \le \mathbb{E} x$.
		\item Lemma 3 from theory \cite{cite:tltheory2010}: the error of a classifier over a distribution $q$ can be bounded by its performance over another distribution $p$ and the distribution divergence: $\mathcal{E}_q(h) \le \mathcal{E}_p(h) + d_{\mathcal{H}\Delta\mathcal{H}}(p, q)$.
	\end{enumerate}
\end{proof}

\myparagraph{Explanation of the upper bound} Theorem~\ref{thm:bound} is intuitively easy to understand. We explain the practical meaning of each term in the upper bound:

\begin{itemize}
	\item $\bar{\delta} = \mathbb{E}_{({\mathbf{x}},y)\sim q} \delta (\hat{\mathbf{y}})$ is the expectation of the compliment of classifier confidence. This term sheds a light on the practice of reducing target uncertainty.

	\item $\mathcal{E}_{I}(h) = \mathbb{E}_{({\mathbf{x}},y)\sim q} \; \mathbf{1}(\phi(\hat{\mathbf{y}}) \notin \mathcal{C})$ is the probability of the classifier predicting a category that is not in the shared class space $\mathcal{C}$. The error of a classifier on the target data can be divided into two parts: Type I error, the prediction is not in $\mathcal{C}$ (thus it is sure to be wrong); Type II error, the prediction is in $\mathcal{C}$ but the predicted label is wrong. The term $\mathcal{E}_{I}(h)$ exactly quantifies the degree of the Type I error. This motivates us to confine the classifier in the shared class space in Section~\ref{sec:class_sel}.

	\item $ \mathcal{E}_{p_{\mathcal{C}}}(h_{\mathcal{C}}) $ is the performance of the classifier in the source domain within the shared class space $\mathcal{C}$. This term is usually small because the classifier is trained on source labeled data supervisedly.

	\item $d_{\mathcal{H}\Delta\mathcal{H}}(p_{\mathcal{C}}, q)$ measures the distribution discrepancy between the source and target domains in $\mathcal{C}$. The discrepancy can be reduced by adversarial domain alignment with instance selection in Section~\ref{sec:instance_sel}.

\end{itemize}

\begin{figure}[htbp]
	\centering
	\subfigure[Data of toy experiment]{\includegraphics[width=.48\columnwidth]{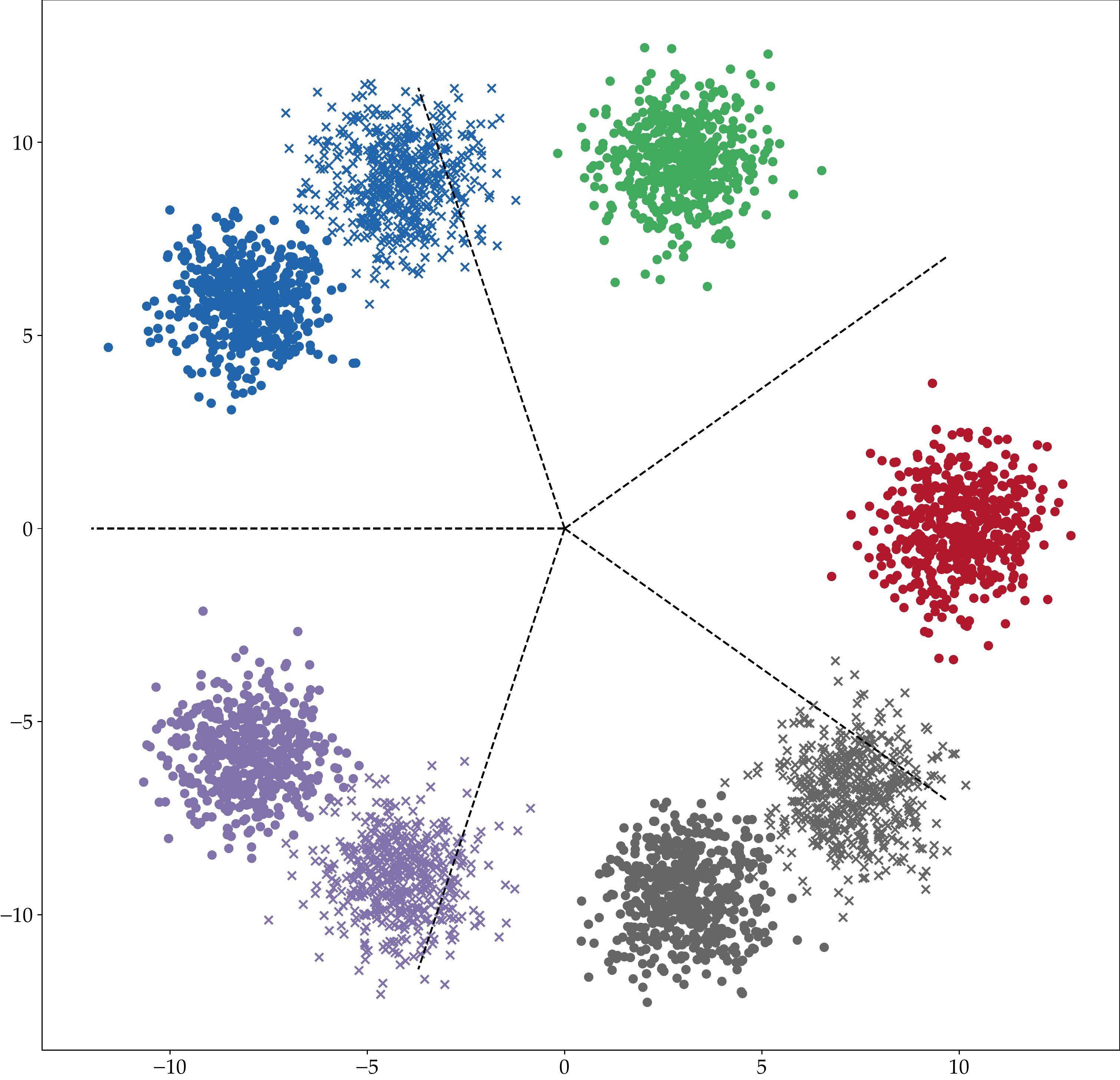}\label{fig:toy_data}}
	\subfigure[Value of terms in the bound]{\includegraphics[width=.48\columnwidth]{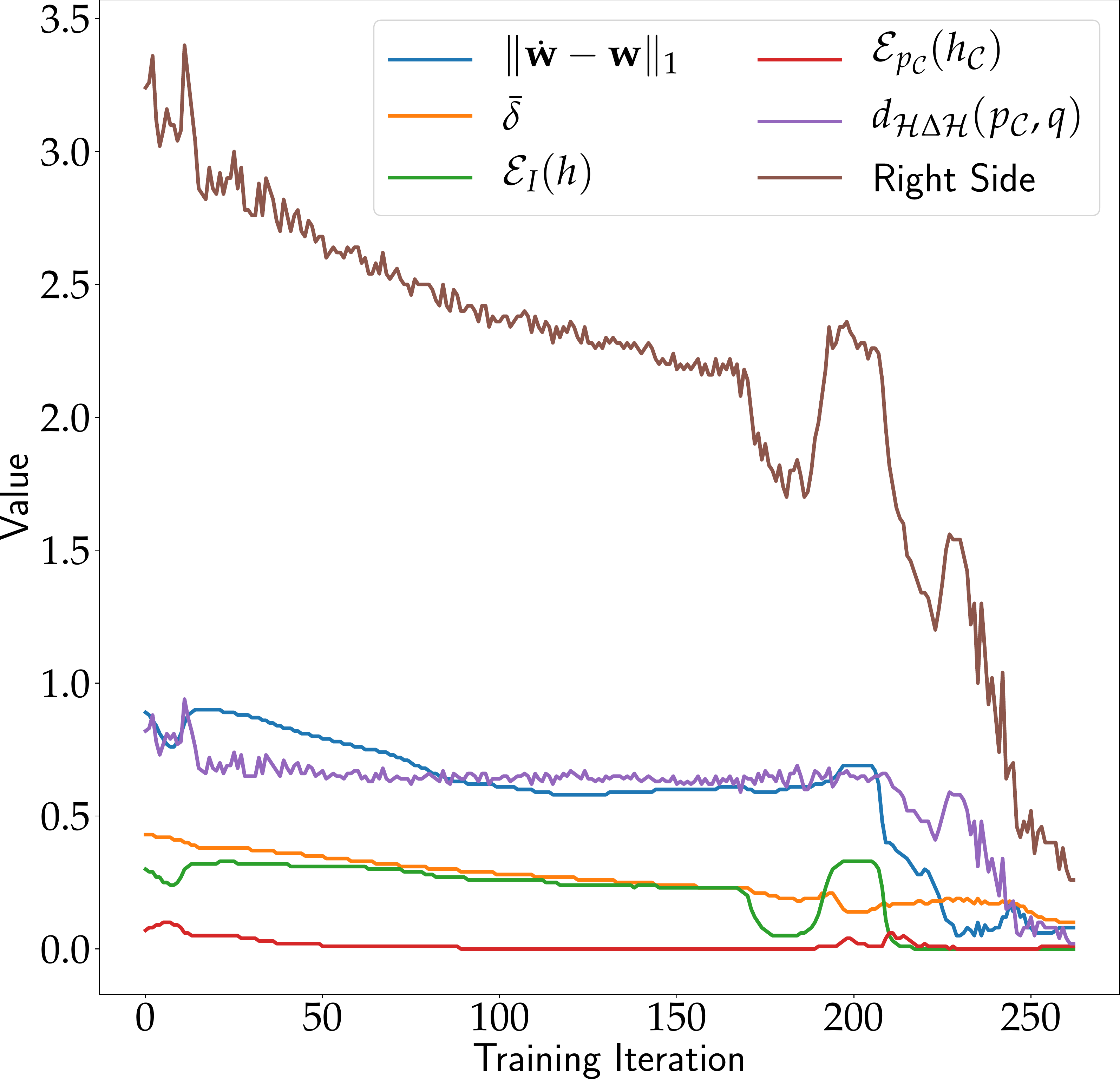}\label{fig:toy_result}}
	\vspace{-5pt}
	\caption{(a) shows the data in the toy experiment. The source domain data are shown by the marker '$\cdot$' and the target domain data are shown by the marker '$\times$'. Different colors indicate the samples of different classes. There are five classes in the source domain and three classes in the target domain, where the source class space subsumes the target class space. The black dashed lines show the optimal decision boundary for the source domain. (b) The change of the value of each term in the bound during the training process, and the sum of the right side in Eqn.~\eqref{eqn:bound}.}
	\label{fig:tightness_bound}
\end{figure}

\myparagraph{The tightness of the upper bound}
Overall, the proposed SAN++ in Section~\ref{sec:sanf} can efficiently reduce all the terms in the upper bound. We simulate the change of the value of each term in Eqn.~\eqref{eqn:bound} in the training process. We collect a dataset consisting of 2D points, where the source domain has $5$ classes and the target domain has $3$ classes (included in the $5$ classes). The distribution of the data is shown in Fig.~\ref{fig:toy_data}. We train SAN++ on this dataset with a two-layer fully-connected network for feature extraction, one fully-connected layer as the classifier and one-layer discriminator for each source class. The values of all the terms are shown in Fig.~\ref{fig:toy_result}. We can observe that the estimation error of the transferable probability on the left side decreases with the decrease of all the terms in the upper bound on the right side. This justifies that tightening the terms on the right side of the bound can substantially reduce the estimation error of the transferable probability on the left side of the bound.

\section{Selective Adversarial Network}
\label{sec:sanf}

Motivated by the theoretical analysis, we present Selective Adversarial Network (SAN) and its improved version (SAN++) with the bi-level selection mechanism. Their architectures are shown in Fig.~\ref{fig:SAN}.

\begin{figure*}[htbp]
	\centering
	\includegraphics[width=0.8\textwidth]{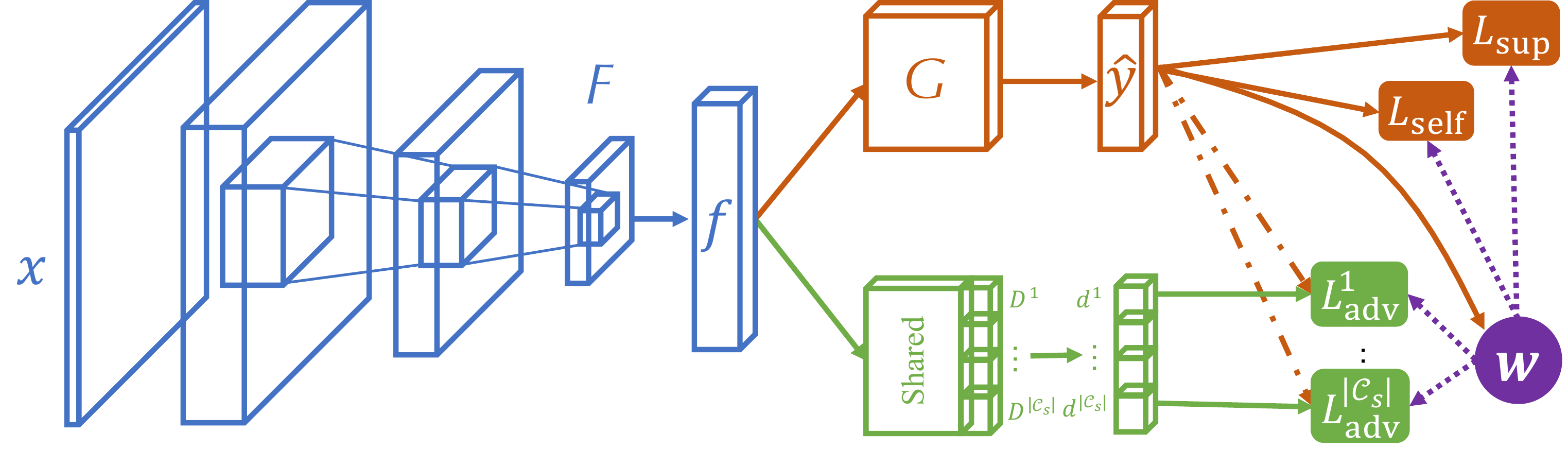}
	\caption{The architecture of Selective Adversarial Network (SAN++) for partial domain adaptation. $F$ is the feature extractor and $\mathbf{f}$ is the feature. $G$ is the classifier and ${\hat{\mathbf{y}}}$ is the predicted label. $D^k|_{k=1}^{|\mathcal{C}_s|}$ is the multi-task discriminator with shared bottom layers and $|\mathcal{C}_s|$ top heads, and ${\hat{d}^k}|_{k=1}^{|\mathcal{C}_s|}$ are the predicted domain labels. $L_\text{sup}$ and $L_\text{self}$, $L_\text{adv}^k|_{k=1}^{|\mathcal{C}_s|}$ are the source supervised training loss, the target self-training loss and the distribution alignment loss respectively. $\mathbf{w}$ denotes the class transferable probability. Solid arrows show data flow. Dashed arrows show instance selection and dotted arrows show class selection.}
	\label{fig:SAN}
\end{figure*}

\subsection{Instance Selection for Positive Transfer}
\label{sec:instance_sel}

\myparagraph{Multitask Discriminator}
We enable distribution alignment in light of the conditional domain adversarial network~\cite{NEURIPS2018_ab88b157}. Because of the class space mismatch, only one domain discriminator to match the whole source and target domains as prior works may cause negative transfer. As shown in Fig.~\ref{fig:SAN}, we instead design a \emph{multi-task discriminator} with shared bottom layers and $|\mathcal{C}_s|$ heads on the top. We use $D^k, {k=1, \ldots, |\mathcal{C}_s|}$ to denote a network chaining the shared layers and the $k$-th head. $D^k$ enables a conditional alignment of the source and target data associated with the $k$-th source class. Our conference version~\cite{cite:CVPR18SAN} uses completely separate discriminators for different classes, which is not parameter-efficient for source dataset of large-scale class space and easy to over-fit since the data in each class are much less than the whole dataset. We empirically show the efficiency of the multitask discriminator with shared bottom layers in Table~\ref{table:memory_time}.

\myparagraph{Instance Selection}
With the multitask discriminator, the main difficulty now is how to assign each target example to the correct domain alignment task, because the target domain is totally unlabeled. As demonstrated by prior works~\cite{saito2017asymmetric}, the predicted probability distribution over the source class space ${\hat{\bf{y}}} = G(F(\mathbf{x}))$ well characterizes the probability of assigning $\mathbf{x}$ to each of the $|\mathcal{C}_s|$ source classes. Thus, it is natural to use ${\hat{\bf{y}}}$ as the \emph{transferable probability} to assign each target instance $\mathbf{x}$ to the $|\mathcal{C}_s|$ alignment tasks. We embed this transferable probability as instance weight on the distribution alignment loss to enable instance selection as follows,
\begin{equation}\label{eqn:Ld}
	\sum\limits_{k = 1}^{|{\mathcal{C}_s}|} {\mathbb{E}_{{\mathbf{x}} \in  {{\mathcal{D}_s} \cup {\mathcal{D}_t}}} \left[{\hat {\bf y}^k \ell_\text{ce}( D^k\left( {{F}\left( {{{\mathbf{x}}}} \right)} ),d \right)} \right] } ,
\end{equation}
where $\ell_\text{ce}$ is the cross-entropy loss, $d$ is the domain label of $\mathbf{x}$ ($0$ for target data and $1$ for source data), and $\hat {\bf y}^k$ is the $k$-th entry of the prediction for $\bf{x}$.

If we had perfect label for each instance, then Eqn.~\eqref{eqn:Ld} will accurately align the feature distributions in a class-correspondence way, and ${d_{\mathcal{H}\Delta\mathcal{H}}}(p_c, q_c)$ will be reduced for any class $c\in\mathcal{C}$ including the term $d_{\mathcal{H}\Delta\mathcal{H}}(p_{\mathcal{C}}, q)$ in Theorem~\ref{thm:bound}. This promotes positive transfer across the source and target domains in the shared class space. The quality of the instance transferable probability can be improved during the training process since the classifier will be made more transferable due to the class-conditional domain alignment.

\subsection{Class Selection against Negative Transfer}
\label{sec:class_sel}

As shown in Section~\ref{sec:theory}, the transferable probability over source classes is derived as $\dot{\mathbf{w}} = {\mathbb{E}_{({\mathbf{x}},y) \sim q}}\dot{\mathbf{y}}$, which reflects the probability that a source class is in the shared class space $\mathcal{C}$. We use this transferable probability to select the classes within the shared class space. Since we have no labels for the target data, we can estimate the transferable probability for each class by the prediction of target data:
\begin{equation}
	\mathbf{w}= {\mathbb{E}_{{\mathbf{x}} \in \mathcal{D}_t}}\hat{\mathbf{y}}  =  {\mathbb{E}_{{\mathbf{x}} \in \mathcal{D}_t}}\left[G(F(\mathbf{x}))\right].
\end{equation}
The error of this estimation can be bounded as Eqn.~\eqref{eqn:bound}, and the improved SAN++, as introduced below, can substantially minimize this upper bound. Thus, $\mathbf{w}$ is an accurate estimate of the true class transferable probability $\dot{\mathbf{w}}$.

\myparagraph{Source Supervised Training}
To reduce the upper bound of the estimation error of $\dot{\mathbf{w}}$ in Eqn.~\eqref{eqn:bound}, we first consider the term $\mathcal{E}_{I}(h) = \mathbb{E}_{({\mathbf{x}},y)\sim q} \; \mathbf{1}(\phi(\hat{\mathbf{y}}) \notin \mathcal{C})$, which reflects the probability of predicting a label in the outlier source class space $\bar{\mathcal{C}}_s$, \emph{i.e.}, the Type I error. To decrease this term, we propose to confine the classifier in the shared class space $\mathcal{C}$ to decrease the training difficulty on source classes in $\bar{\mathcal{C}}_s$. This further reduces the probability of predicting a data point into $\bar{\mathcal{C}}_s$. Specifically, we incorporate the transferable probability $\mathbf{w}$ into the source classifier $G$ and derive a new selective classification loss as follows,
\begin{equation}\label{eqn:classifier}
	L_\text{sup} = \mathbb{E}_{({\mathbf{x}},y) \in {\mathcal{D}_s}} \left[{\mathbf{w}_{y} \bullet {\ell_\text{ce}}\left( {{G}\left( {{F}\left( {{\mathbf{x}}} \right)} \right)}, y \right)}\right],
\end{equation}
where $\mathbf{w}_{ y}$ is the $ y$-th entry of $\mathbf{w}$, indicating the probability of a label $y$ falling into the shared class space $\mathcal{C}$. The class selection on the source classifier enforces the model to focus on the source data from $\mathcal{C}$, which is well-motivated from our theoretical results to substantially reduce the estimation error of the class transferable probability.

\myparagraph{Target Self-training}
To reduce the term $\bar{\delta} = \mathbb{E}_{({\mathbf{x}},y)\sim q} \delta (\hat{\mathbf{y}})$, we need to reduce the uncertainty of target predictions. Self-training is a learning technique to reduce the prediction uncertainty, which iteratively assigns pseudo-labels to unlabeled data and re-trains the model with both labeled data and pseudo-labeled data. Self-training~\cite{xie2020self} is demonstrated to have fewer hyperparameters and easier to tune than other widely-used techniques to reduce uncertainty such as entropy minimization~\cite{cite:NIPS04SSLEM,cite:NIPS16RTN,cite:CVPR18SAN,cite:CVPR18IWAN}. Therefore, instead of using the entropy minimization as in our conference version~\cite{cite:CVPR18SAN}, we utilize self-training to regularize the predictions of the target domain. Different from prior self-training works, we firstly incorporate class selection into self-training and confine self-training into the shared class space $\mathcal{C}$ with the class transferable probability. This is because applying self-training on all the target data suffers from the Type I error introduced above. Specifically, after the convergence of the training on source labeled data, we assign soft pseudo-labels $\hat{y}_\text{self}$ to the target unlabeled data $\mathbf{x}$ with the classifier $G(F(\bullet))$ and then incorporate them into training:
\begin{equation}\label{eqn:self-training}
	L_\text{self} = \mathbb{E}_{{{\mathbf{x}}} \in {\mathcal{D}_t}} \left[\mathbf{w}_{\hat{y}_\text{self}} \bullet {\ell_\text{ce}}\left( {{G}\left( {{F}\left( {{\mathbf{x}}} \right)} \right)}, \hat{y}_\text{self} \right)\right].
\end{equation}
With a fixed number of epochs ($1$ in all of our experiments), we repeat the pseudo-labeling and self-training process to use the better SAN++ model to assign more accurate pseudo-labels. Self-training improves the prediction confidence and generates more accurate pseudo-labels without any hyperparameters, \emph{e.g.} confidence threshold, which makes $\bar{\delta} $ smaller and lowers the upper bound in Theorem~\ref{thm:bound}.

\myparagraph{Source-Target Adversarial Adaptation}
Section~\ref{sec:instance_sel} introduces the instance selection with a multi-task discriminator to reduce the distribution discrepancy, which reduces the term $d_{\mathcal{H}\Delta\mathcal{H}}(p_{\mathcal{C}}, q)$. However, only applying the instance selection will train all the distribution alignment tasks, undesirably including the tasks for source classes in $\bar{\mathcal{C}_s}$. However, for an outlier class $k \in \bar{\mathcal{C}_s}$, the corresponding discriminator $D^k$ just receives positive examples (source examples) and provides useless signals for the partial domain adaptation task. Such learning signals also update the feature extractor $F$, which may worsen the performance of the alignment tasks in the shared class space $\mathcal{C}$. Thus, we use class transferable probability $\mathbf{w}$ to alleviate the outlier influence. {Similar to CDAN~\cite{NEURIPS2018_ab88b157}, we prioritize the
discriminator on those easy-to-transfer examples with certain predictions by reweighting each training example of the conditional domain discriminator by an entropy-aware weight $w_e(\mathbf{x})=1+e^{-H(G(F(\mathbf{x})))}$, where $H$ is the entropy function. Thus, the final loss for source-target adversarial adaptation is defined as follows:
\begin{equation}\label{eqn:Ld2}
	{L_\text{adv}} = \sum\limits_{k = 1}^{|{\mathcal{C}_s}|} \mathbb{E}_{{\mathbf{x}} \in  {{\mathcal{D}_s} \cup {\mathcal{D}_t}} } \left[\mathbf{w}_k \bullet w_e(\mathbf{x}) \hat {\bf y}^k \ell_\text{ce} ( {D^k\left( {{G}\left( {{{\mathbf{x}}}} \right)} \right),d} ) \right].
\end{equation}}
Applying the class transferable probability on the distribution alignment losses of different tasks can make the model focus on the distribution alignment tasks of classes in the shared class space $\mathcal{C}$. This correctly promotes the positive transfer and decreases or even removes the influence of the tasks responsible for outlier source class space $\bar{\mathcal{C}_s}$.

\subsection{Selective Adversarial Partial Domain Adaptation}
\label{sec:conclusion}

We use $ \theta_F $ to denote the parameters of the feature extractor $F$, $ \theta_G $ to denote the parameters of the label predictor $G$, $ \theta_{D} $ to denote all the parameters of the multitask discriminator including the shared bottom layers and the multiple top heads. Integrating all modules and losses, the objective of the Selective Adversarial Network (SAN++) is
\begin{equation}\label{eqn:MultiA}
	O\left( {\theta _F},{\theta_G},\theta_D \right) = L_\text{sup} + L_\text{self} - L_\text{adv}.
\end{equation}
The optimization problem is a min-max game to find the optimal parameters ${\hat\theta_F}$, ${\hat\theta_G}$ and ${\hat\theta_{D}}$ that satisfy
\begin{equation}\label{eqn:parameter1}
	\begin{aligned}
		({\hat\theta_F}, {\hat\theta_G}) & = \mathop {\arg \min }\limits_{{\theta _F},{\theta _G}} O\left( {{\theta _F},{\theta _G},\theta _{D}} \right), \\
		({\hat\theta_{D}})               & = \mathop {\arg \max }\limits_{{\theta_{D}}} O\left( {{\theta_F},{\theta _G},\theta_{D}} \right).
	\end{aligned}
\end{equation}
SAN++ is developed for partial domain adaptation with theoretical guarantees, which jointly circumvents negative transfer by ignoring the outlier source class space $\bar{\mathcal{C}_s}$ and its affiliated data, and promotes positive transfer by conditionally matching data distributions $p_{\mathcal{C}}$ and $q$ in the shared class space $\mathcal{C}$. SAN++ empowers a positive interplay between the estimation of the transferable probability and the learning of the partial domain adaptation model.

\section{Experiments on PDA}
We conduct extensive experiments on seven datasets to fully evaluate our approach for a variety of partial domain adaptation problems. All codes and datasets are available online at \url{https://github.com/thuml/Transfer-Learning-Library}.

\subsection{Datasets and Settings}

Six datasets are used in this paper. Table~\ref{tab:dataset} shows details about tasks and the target class space for each dataset.

\medskip\noindent\textbf{Office-31}~\cite{cite:ECCV10Office} consists of $4,652$
images in $31$ categories from three domains: \textit{Amazon} (\textbf{A}), \textit{Webcam} (\textbf{W}) and \textit{DSLR} (\textbf{D}). This dataset represents the performance on the setting where both source and target domains have a small class space.

\medskip\noindent\textbf{Office-Home}~\cite{cite:CVPR17OfficeHome} consists of 4 different domains: Artistic images (\textbf{Ar}), Clipart images (\textbf{Cl}), Product images (\textbf{Pr}) and Real-World images (\textbf{Rw}) in from $65$ categories. The dataset represents domains with different visual artifacts.

\medskip\noindent\textbf{ImageNet-Caltech} is constructed from \textit{ImageNet-1K}~\cite{cite:ILSVRC15} and \textit{Caltech-256}~\cite{griffin_caltech-256_2007}, which share $84$ common categories. Because pre-trained models are pre-trained on ImageNet training set, we adopt the validation set when ImageNet is used as the target domain and adopt the training set when ImageNet is used as the source domain. This setting represents partial domain adaptation with large-scale source class space.

\medskip\noindent\textbf{VisDA-2017}~\cite{peng_visda:_2018} has two domains: one consists of synthetic 2D renderings of 3D models and the other consists of real images associated with $12$ categories. This dataset aims to encourage the use of easily-accessible computer-generated data to improve performance with less labeling cost.

\medskip\noindent\textbf{Open MIC} ~\cite{koniusz2018museum} dataset contains photos of exhibits captured in $10$ distinct exhibition spaces of several museums. The source photos are taken in a controlled fashion while the target photos are taken by wearable cameras in an egocentric setup. We use the official $5$ target splits of train, validation and test and report the average performance. The dataset exhibits a more challenging classification task with a larger domain shift including motion blur, occlusions, varying viewpoints, color inconstancy, etc.

\medskip\noindent\textbf{Digits} dataset is composed of three standard digit classification datasets: MNIST~\cite{cite:IEEE98MNIST}, USPS~\cite{cite:TPAMI94USPS} and SVHN~\cite{cite:NIPS11SVHN}. Each dataset contains $10$ categories, ranging from $0$ to $9$. MNIST and USPS include grayscale handwritten digits captured under constrained conditions. SVHN dataset was constructed by cropping house numbers in RGB Google Street View images, which has more diversity.

\begin{table}[tbp]
	\centering
	\renewcommand\tabcolsep{1pt}
	\caption{Details of the datasets, the tasks, and the target class space for all PDA experiments. For each dataset, we use all classes as the source class space.}
	\label{tab:dataset}
	\resizebox{1\columnwidth}{!}{
		\begin{tabular}{ccc}
			\toprule
			Dataset                     & Tasks                                                                    & Target Class Space                           \\
			\midrule
			\makecell{Office-31} & \makecell{\textbf{A} $\rightarrow$ \textbf{W}, \textbf{D} $\rightarrow$ \textbf{W}, \\ \textbf{W} $\rightarrow$ \textbf{D}, \textbf{A} $\rightarrow$ \textbf{D}, \\ \textbf{D} $\rightarrow$ \textbf{A}, \textbf{W} $\rightarrow$ \textbf{A}} & \makecell{$10$ categories shared with \\ Caltech-256} \\
			\midrule
			\makecell{Office-Home} & \makecell{\textbf{Ar} $\rightarrow$ \textbf{Cl}, \textbf{Ar} $\rightarrow$ \textbf{Pr},\\ \textbf{Ar} $\rightarrow$ \textbf{Rw}, \textbf{Cl} $\rightarrow$ \textbf{Ar},\\ \textbf{Cl} $\rightarrow$ \textbf{Pr}, \textbf{Cl} $\rightarrow$ \textbf{Rw},\\ \textbf{Pr} $\rightarrow$ \textbf{Ar}, \textbf{Pr} $\rightarrow$ \textbf{Cl},\\ \textbf{Pr} $\rightarrow$ \textbf{Rw}, \textbf{Rw} $\rightarrow$ \textbf{Ar},\\ \textbf{Rw} $\rightarrow$ \textbf{Cl}, \textbf{Rw} $\rightarrow$ \textbf{Pr}} &  \makecell{the first $25$ categories \\ in alphabetic order} \\
			\midrule
			\makecell{ImageNet-Caltech} & \textbf{I} $\rightarrow$ \textbf{C}, \textbf{C} $\rightarrow$ \textbf{I} & \makecell{$84$ categories shared by          \\ ImageNet-1K and Caltech-256}\\
			\midrule
			\makecell{VisDA-2017}       & \textbf{VisDA-2017}                                                      & \makecell{the first $6$ categories           \\ in alphabetic order}  \\
			\midrule
			\\[-5pt]
			\makecell{OpenMIC}          & \textbf{OpenMIC}                                                         & \makecell{'Clk', 'Clv', 'Gls', 'Hon', 'Nat'} \\
			\\[-5pt]
			\midrule
			\makecell{Digits} & \makecell{\textbf{SVHN} $\rightarrow$ \textbf{MNIST},\\ \textbf{MNIST} $\rightarrow$ \textbf{USPS},\\ \textbf{USPS} $\rightarrow$ \textbf{MNIST}} & '0','1','2','3'\\
			\bottomrule
		\end{tabular}}
\end{table}

\subsection{Implementation Details}

To demonstrate the effectiveness of our approach, we compare it with the source-only  baseline \textbf{ResNet}~\cite{cite:CVPR16DRL}, deep domain adaptation methods: Residual Transfer
Network (\textbf{RTN})~\cite{cite:NIPS16RTN}, Joint Adaptation
Network (\textbf{JAN})~\cite{cite:ICML17JAN}, Central Moment Discrepancy (\textbf{CMD})~\cite{cite:ICLR17CMD}, Domain Adversarial Neural Network (\textbf{DANN})~\cite{cite:ICML15DANN}, Adversarial Discriminative Domain Adaptation (\textbf{ADDA})~\cite{cite:CVPR17ADDA}, and partial domain adaptation methods: Importance Weighted Adversarial Nets (\textbf{IWAN})~\cite{cite:CVPR18IWAN}, Partial Adversarial Domain Adaptation (\textbf{PADA})~\cite{cao_partial_2018}, Deep Residual Correction Network (\textbf{DRCN})~\cite{li2020deep} and Example Transfer Network (\textbf{ETN})~\cite{cao_learning_2019}. We also compare with our conference version (\textbf{SAN})~\cite{cite:CVPR18SAN}, and use \textbf{SAN++} to refer to our journal version of SAN.

We implement all deep methods based on \textbf{PyTorch}, and fine-tune from ResNet-50~\cite{cite:CVPR16DRL} pre-trained on ImageNet. We add a bottleneck layer before the classifier layer as DANN \cite{cite:ICML15DANN} except for the task I$\rightarrow$C because the task can fully exploit the advantage of the original feature and classifier layers in the ImageNet pre-trained model. We fine-tune all the feature layers and train the bottleneck layer, the classifier layer and the multitask discriminator. Since the new layers are trained from scratch, we set their learning rate to be $10$ times of the pre-trained layers. We use mini-batch stochastic gradient descent (SGD) with momentum of $0.9$ and the learning rate annealing strategy of the following formula: $\eta_p = \frac{\eta_0}{{(1+\alpha p)}^\beta}$ as DANN, where $p$ is the training progress linearly changing from $0$ to $1$, $\eta_0 $ is the initial learning rate, $\alpha = 10$ and $\beta = 0.75$. The penalty of adversarial learning is increased from $0$ to $1$ gradually as DANN. Initial learning rates are selected by Deep Embedded Validation~\cite{you_towards_2019}. For MMD-based methods (RTN and JAN), we use Gaussian kernel with bandwidth set to median pairwise squared distances on training data, \emph{i.e.} median heuristic~\cite{cite:NIPS12MKMMD}.

\subsection{Results}
\label{sec:results}

This section shows experimental results on closed-set domain adaptation and partial domain adaptation. Although our methods (SAN/SAN++) are designed for partial domain adaptation, they can also deal with closed-set domain adaptation, which is a special case of partial domain adaptation.

\begin{table}[htbp]
	\addtolength{\tabcolsep}{-5pt}
	\centering
	\caption{Accuracy (\%) of the closed-set domain adaptation tasks on \emph{Office-31} (ResNet-50).}
	\label{table:close_accuracy_office}
	\resizebox{1\columnwidth}{!}{
		\begin{tabular}{lccccccc}
			\toprule
			\multirow{2}{30pt}{Method} & \multicolumn{7}{c}{Office-31} \\
			\cmidrule{2-8}
			                                 & A $\rightarrow$ W & D $\rightarrow$ W & W $\rightarrow$ D & A $\rightarrow$ D & D $\rightarrow$ A & W $\rightarrow$ A & Avg            \\
			\midrule
			ResNet~\cite{cite:CVPR16DRL}     & 68.43             & 96.75             & 99.36             & 68.95             & 62.54             & 60.72             & 76.14          \\
			RTN~\cite{cite:NIPS16RTN}        & 84.56             & 96.85             & 99.48             & 77.59             & 66.26             & 64.88             & 81.67          \\
			JAN~\cite{cite:ICML17JAN}        & 85.45             & 97.46             & 99.84             & 84.77             & 68.63             & 70.05             & 84.35          \\
			CMD~\cite{cite:ICLR17CMD}        & 84.03             & 98.24             & 99.26             & 86.62             & \textbf{72.81}    & 70.31             & 85.21          \\
			DANN~\cite{cite:ICML15DANN}      & 82.03             & 96.95             & 99.14             & 79.76             & 68.25             & 67.47             & 82.25          \\
			ADDA~\cite{cite:CVPR17ADDA}      & 86.22             & 96.21             & 98.49             & 77.84             & 69.52             & 68.96             & 82.94          \\
			CDAN~\cite{NEURIPS2018_ab88b157} & 94.13             & 98.62             & \textbf{100.00}   & 92.94             & 70.99             & 69.28             & 87.71          \\
			\midrule
			SAN++ w/ DANN                    & 92.94             & 98.56             & \textbf{100.00}   & 92.78             & 71.92             & 70.65             & 87.81          \\
			{SAN++ w/ CDAN}         & \textbf{95.23}    & \textbf{98.79}    & \textbf{100.00}   & \textbf{94.01}    & {72.45}           & \textbf{71.87}    & \textbf{88.73} \\
			\bottomrule
		\end{tabular}
	}
\end{table}

\begin{table*}[h]
	\addtolength{\tabcolsep}{-3pt}
	\centering
	\captionsetup{justification=centering}
	\caption{Accuracy (\%) of partial domain adaptation tasks on \emph{Office-Home} (ResNet-50).}
	\vspace{-5pt}
	\label{table:accuracy_officehome}
	\resizebox{\textwidth}{!}{%
		\begin{tabular}{lccccccccccccc}
			\toprule
			\multirow{2}{30pt}{Method} & \multicolumn{13}{c}{Office-Home} \\
			\cmidrule{2-14}
			                             & {Ar}$\rightarrow${Cl} & {Ar}$\rightarrow${Pr} & {Ar}$\rightarrow${Rw} & {Cl}$\rightarrow${Ar} & {Cl}$\rightarrow${Pr} & {Cl}$\rightarrow${Rw} & {Pr}$\rightarrow${Ar} & {Pr}$\rightarrow${Cl} & {Pr}$\rightarrow${Rw} & {Rw}$\rightarrow${Ar} & {Rw}$\rightarrow${Cl} & {Rw}$\rightarrow${Pr} & Avg            \\
			\midrule
			ResNet~\cite{cite:CVPR16DRL} & 46.33                 & 67.51                 & 75.87                 & 59.14                 & 59.94                 & 62.73                 & 58.22                 & 41.79                 & 74.88                 & 67.40                 & 48.18                 & 74.17                 & 61.35          \\
			RTN~\cite{cite:NIPS16RTN}    & 49.31                 & 57.70                 & 80.07                 & 63.54                 & 63.47                 & 73.38                 & 65.11                 & 41.73                 & 75.32                 & 63.18                 & 43.57                 & 80.50                 & 63.07          \\
			JAN~\cite{cite:ICML17JAN}    & 47.61                 & 58.66                 & 67.84                 & 50.81                 & 53.06                 & 61.33                 & 53.11                 & 44.87                 & 68.56                 & 61.00                 & 51.79                 & 69.53                 & 57.35          \\
			DANN~\cite{cite:ICML15DANN}  & 43.76                 & 67.90                 & 77.47                 & 63.73                 & 58.99                 & 67.59                 & 56.84                 & 37.07                 & 76.37                 & 69.15                 & 44.30                 & 77.48                 & 61.72          \\
			IWAN~\cite{cite:CVPR18IWAN}  & 53.94                 & 54.45                 & 78.12                 & 61.31                 & 47.95                 & 63.32                 & 54.17                 & 52.02                 & 81.28                 & 76.46                 & 56.75                 & 82.90                 & 63.56          \\
			DRCN~\cite{li2020deep}       & 51.60                 & 75.80                 & 82.00                 & 62.90                 & 65.10                 & 72.90                 & 67.40                 & 50.00                 & 81.00                 & 76.40                 & {57.70}               & 79.30                 & 68.50          \\
			ETN~\cite{cao_learning_2019} & {59.24}               & {77.03}               & 79.54                 & 62.92                 & 65.73                 & 75.01                 & {68.29}               & {55.37}               & 84.37                 & 75.72                 & 57.66                 & {84.54}               & 70.45          \\
			\midrule
			SAN~\cite{cite:CVPR18SAN}    & 44.42                 & 68.68                 & 74.60                 & 67.49                 & 64.99                 & 77.80                 & 59.78                 & 44.72                 & 80.07                 & 72.18                 & 50.21                 & 78.66                 & 65.30          \\
			\midrule
			SAN++                        & \textbf{61.25}        & \textbf{81.57}        & \textbf{88.57}        & \textbf{72.82}        & \textbf{76.41}        & \textbf{81.94}        & \textbf{74.47}        & \textbf{57.73}        & \textbf{87.24}        & \textbf{79.71}        & \textbf{63.76}        & \textbf{86.05}        & \textbf{75.96} \\
			\bottomrule
		\end{tabular}
	}
\end{table*}

\subsubsection{Closed-Set Domain Adaptation}
\label{sec:exp_close}

We first show that SAN++ performs comparably to prior works for closed-set domain adaptation. We conduct the experiment on Office-31. {We use both DANN~\cite{cite:ICML15DANN} and CDAN~\cite{NEURIPS2018_ab88b157} to perform distribution alignment for SAN++ in this experiment and we conduct experiments on the Office-31 dataset.} Since the other partial domain adaptation methods usually focus on the problem of outlier source class space, we do not compare with them in this experiment.

{As shown in Table~\ref{table:close_accuracy_office}, we can observe that in closed-set domain adaptation, SAN++ with DANN and with CDAN both perform comparably to the corresponding domain adaptation methods that use a similar distribution alignment module. The results demonstrate that the mechanism to filter out outlier source samples in SAN++ does not hurt its performance for closed-set domain adaptation.}

\begin{table}[htbp]
	\addtolength{\tabcolsep}{-4.1pt}
	\centering
	\caption{Accuracy (\%) of partial domain adaptation tasks on \emph{Office-31} (ResNet-50).}
	\vspace{-5pt}
	\label{table:accuracy_office}
	\resizebox{1\columnwidth}{!}{
		\begin{tabular}{lccccccc}
			\toprule
			\multirow{2}{30pt}{Method} & \multicolumn{7}{c}{Office-31} \\
			\cmidrule{2-8}
			                             & A $\rightarrow$ W & D $\rightarrow$ W & W $\rightarrow$ D & A $\rightarrow$ D & D $\rightarrow$ A & W $\rightarrow$ A & Avg            \\
			\midrule
			ResNet~\cite{cite:CVPR16DRL} & 75.59             & 96.27             & 98.09             & 83.44             & 83.92             & 84.97             & 87.05          \\
			RTN~\cite{cite:NIPS16RTN}    & 78.98             & 93.22             & 85.35             & 77.07             & 89.25             & 89.46             & 85.56          \\
			JAN~\cite{cite:ICML17JAN}    & 80.32             & 94.55             & 89.23             & 84.56             & 90.74             & 90.23             & 88.27          \\
			CMD~\cite{cite:ICLR17CMD}    & 79.12             & 93.98             & 90.12             & 81.43             & 87.65             & 89.01             & 86.89          \\
			DANN~\cite{cite:ICML15DANN}  & 73.56             & 96.27             & 98.73             & 81.53             & 82.78             & 86.12             & 86.50          \\
			ADDA~\cite{cite:CVPR17ADDA}  & 75.67             & 95.38             & 99.85             & 83.41             & 83.62             & 84.25             & 87.03          \\
			IWAN~\cite{cite:CVPR18IWAN}  & 89.15             & 99.32             & 99.36             & 90.45             & 95.62             & 94.26             & 94.69          \\
			PADA~\cite{cao_partial_2018} & 86.54             & 99.32             & \textbf{100.00}   & 82.17             & 92.69             & 95.41             & 92.69          \\
			DRCN~\cite{li2020deep}       & 90.80             & \textbf{100.00}   & \textbf{100.00}   & 94.30             & 95.20             & 94.80             & 95.90          \\
			ETN~\cite{cao_learning_2019} & 94.52             & \textbf{100.00}   & \textbf{100.00}   & 95.03             & \textbf{96.21}    & 94.64             & 96.73          \\
			\midrule
			SAN~\cite{cite:CVPR18SAN}    & 93.90             & 99.32             & 99.36             & 94.27             & 94.15             & 88.73             & 94.96          \\
			\midrule
			SAN++                        & \textbf{99.66}    & \textbf{100.00}   & \textbf{100.00}   & \textbf{98.09}    & 94.05             & \textbf{95.51}    & \textbf{97.89} \\
			\bottomrule
		\end{tabular}}
\end{table}

\begin{table}[htbp]
	\addtolength{\tabcolsep}{3pt}
	\centering
	\caption{Accuracy (\%) of partial domain adaptation tasks on \emph{VisDA-2017}, \emph{ImageNet-Caltech} and \emph{OpenMIC} (ResNet-50).}
	\vspace{-5pt}
	\label{table:accuracy_visic}
	\resizebox{1\columnwidth}{!}{
		\begin{tabular}{lcccc}
			\toprule
			\multirow{2}{30pt}{Method} & \multirow{2}{30pt}{\centering VisDA2017} & \multicolumn{2}{c}{ImageNet-Caltech} & \multirow{2}{30pt}{Open MIC}\\
			\cmidrule(lr){3-4}
			                             &                & I $\rightarrow$ C & C $\rightarrow$ I &                \\
			\midrule
			ResNet~\cite{cite:CVPR16DRL} & 45.26          & 69.69             & 71.29             & 71.45          \\
			RTN~\cite{cite:NIPS16RTN}    & 50.04          & 75.50             & 66.21             & 61.34          \\
			JAN~\cite{cite:ICML17JAN}    & 50.38          & 72.16             & 62.47             & 63.91          \\
			DANN~\cite{cite:ICML15DANN}  & 51.01          & 70.80             & 67.71             & 57.01          \\
			IWAN~\cite{cite:CVPR18IWAN}  & 48.62          & 78.06             & 73.33             & 71.70          \\
			PADA~\cite{cao_partial_2018} & 53.53          & 75.03             & 70.48             & 64.56          \\
			DRCN~\cite{li2020deep}       & 58.20          & 75.30             & 78.90             & 74.67          \\
			ETN~\cite{cao_learning_2019} & 57.86          & 83.23             & 74.93             & 75.54          \\
			\midrule
			SAN~\cite{cite:CVPR18SAN}    & 49.92          & 77.75             & 75.26             & 73.49          \\
			\midrule
			SAN++                        & \textbf{63.06} & \textbf{83.34}    & \textbf{81.07}    & \textbf{77.56} \\
			\bottomrule
		\end{tabular}}
\end{table}

\subsubsection{Partial Domain Adaptation}
\label{sec:exp_pda}

We then compare the performance of different methods for partial domain adaptation. {Prior works show that more advanced distribution alignment method, CDAN, performs worse than DANN on partial domain adaptation tasks~\cite{liang2020baus}, so we only use DANN as the distribution alignment module in this experiment.} We conduct experiments on a wide range of datasets with different kinds of domain shifts. We show the classification accuracy on Office-Home in Table~\ref{table:accuracy_officehome}, Office-31 in Table~\ref{table:accuracy_office} and VisDA/ImageNet-Caltech/OpenMIC in Table~\ref{table:accuracy_visic}. We can observe that SAN++ consistently outperforms all the domain adaptation and partial domain adaptation methods on nearly all the datasets. In particular, SAN++ outperforms the state-of-the-art partial domain adaptation methods: DRCN and ETN, which justifies that SAN++ maximally downweights the outlier-class examples and upweights the shared-class examples. By introducing the class selection on all training losses of source supervised training, target self-training and source-target adversarial adaptation, SAN++ significantly outperforms SAN.

In addition, we find that some closed-set domain adaptation methods perform even worse than ResNet trained solely on source data. This is because they all assume identical class space and match the whole source distribution with the target distribution. The assumption does not hold in partial domain adaptation and negative transfer degrades their performance.

\begin{table}[htbp]
	\addtolength{\tabcolsep}{-3pt}
	\centering
	\caption{Accuracy (\%) of SAN++, SAN and pixel-level domain adaptation method on \emph{Digits}.}
	\vspace{-5pt}
	\label{table:digit}
	\resizebox{\columnwidth}{!}{
		\begin{tabular}{lccc}
			\toprule
			Method                              & USPS $\rightarrow$ MNIST & MNIST  $\rightarrow$ USPS & SVHN  $\rightarrow$ MNIST \\
			\midrule
			GTA~\cite{cite:CVPR18GenerateAdapt} & 50.47                    & 33.72                     & 56.71                     \\
			SAN~\cite{cite:CVPR18SAN}           & 91.28                    & 94.56                     & 66.84                     \\
			SAN++                               & \textbf{94.65}           & \textbf{99.23}            & \textbf{68.47}            \\
			\bottomrule
		\end{tabular}}
\end{table}

We further compare SAN++ with GTA~\cite{cite:CVPR18GenerateAdapt}, the state-of-the-art \emph{pixel-level} domain adaptation method, on the Digits dataset in Table~\ref{table:digit}. Although GTA is strong on closed-set domain adaptation tasks especially on the Digits dataset, SAN++ outperforms GTA on all the three partial domain adaptation tasks by a large margin. This confirms that GTA cannot mitigate the negative transfer difficulty.

\begin{figure}[htbp]
	\centering
	\includegraphics[width=.8\columnwidth]{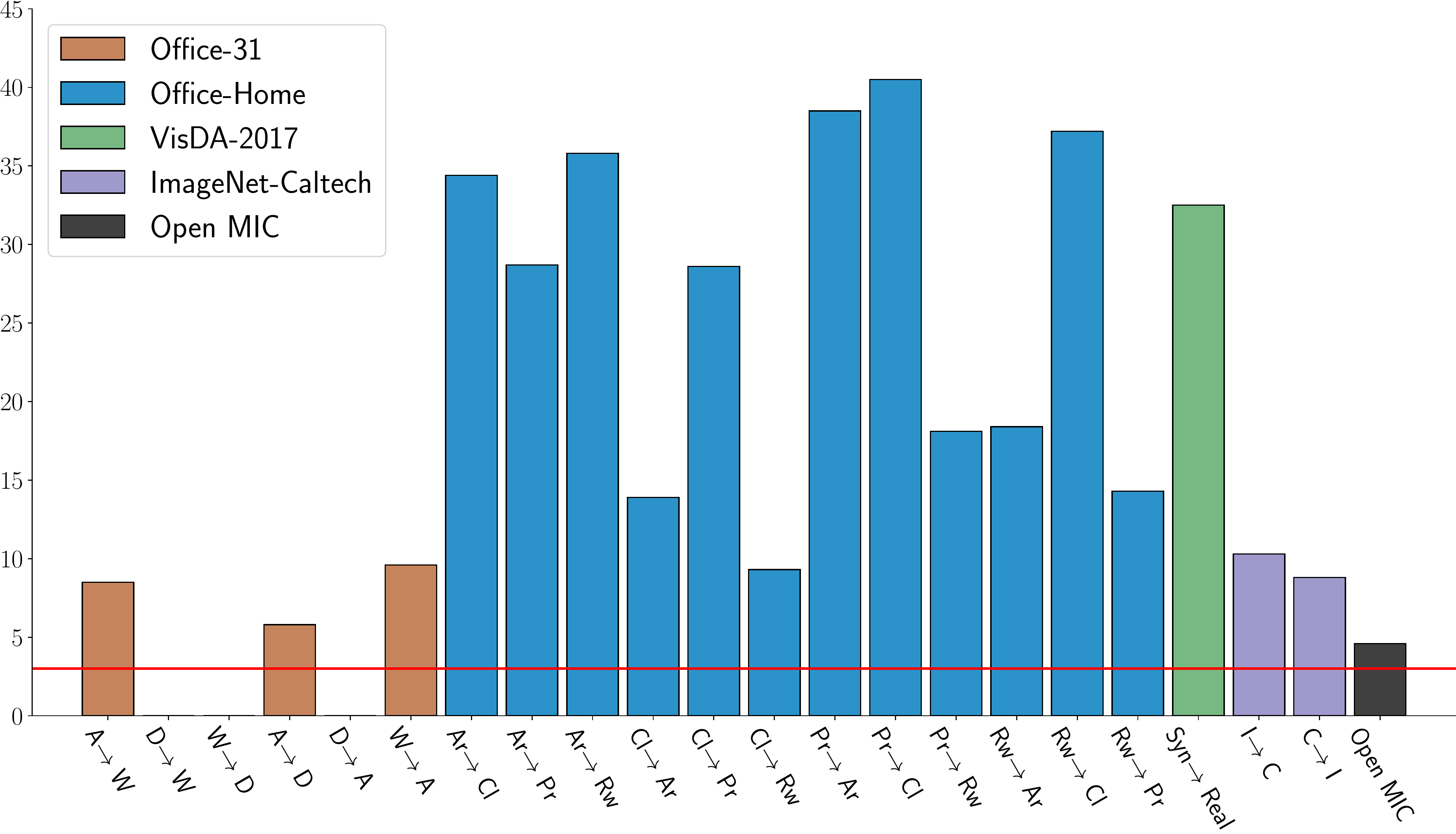}
	\caption{The negative natural logarithm of p-value that SAN++ outperforms SAN. The red line is $-\log(0.05)$ indicating the bar of statistical significance.}
	\label{fig:sig_san}
\end{figure}

The statistical significance that SAN++ outperforms SAN and ETN are shown in Fig.~\ref{fig:sig_san} and Fig.~\ref{fig:sig_etn} respectively. The red line indicates the p-value $0.05$, where the values above this line are statistically significant. We observe that SAN++ significantly outperforms SAN on most tasks. Comparing with the state-of-the-art ETN, SAN++ still achieves statistically significant improvement on half of the tasks.

\begin{figure}[htbp]
	\centering
	\includegraphics[width=.8\columnwidth]{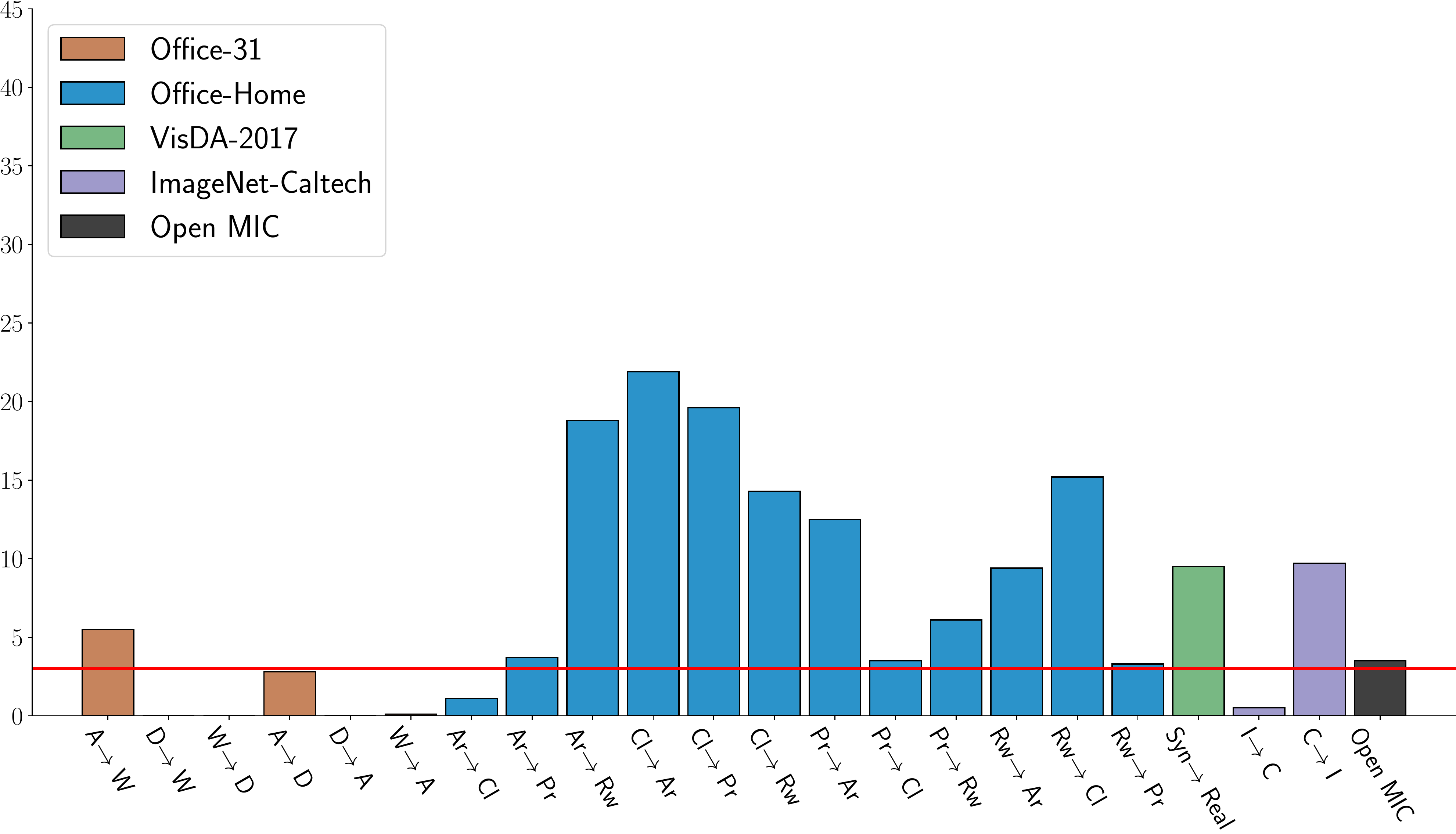}
	\caption{The negative natural logarithm of p-value that SAN++ outperforms ETN.  The red line is $-\log(0.05)$ indicating the bar of statistical significance.}
	\label{fig:sig_etn}
\end{figure}

\subsection{Analyses}

The previous section shows that SAN++ is a very effective algorithm for partial domain adaptation. Now we present additional experiments to delve into the effectiveness of the proposed SAN++ algorithm from several perspectives.

\begin{figure}[htbp]
	\centering
	\subfigure[Accuracy]{\includegraphics[width=0.23\textwidth]{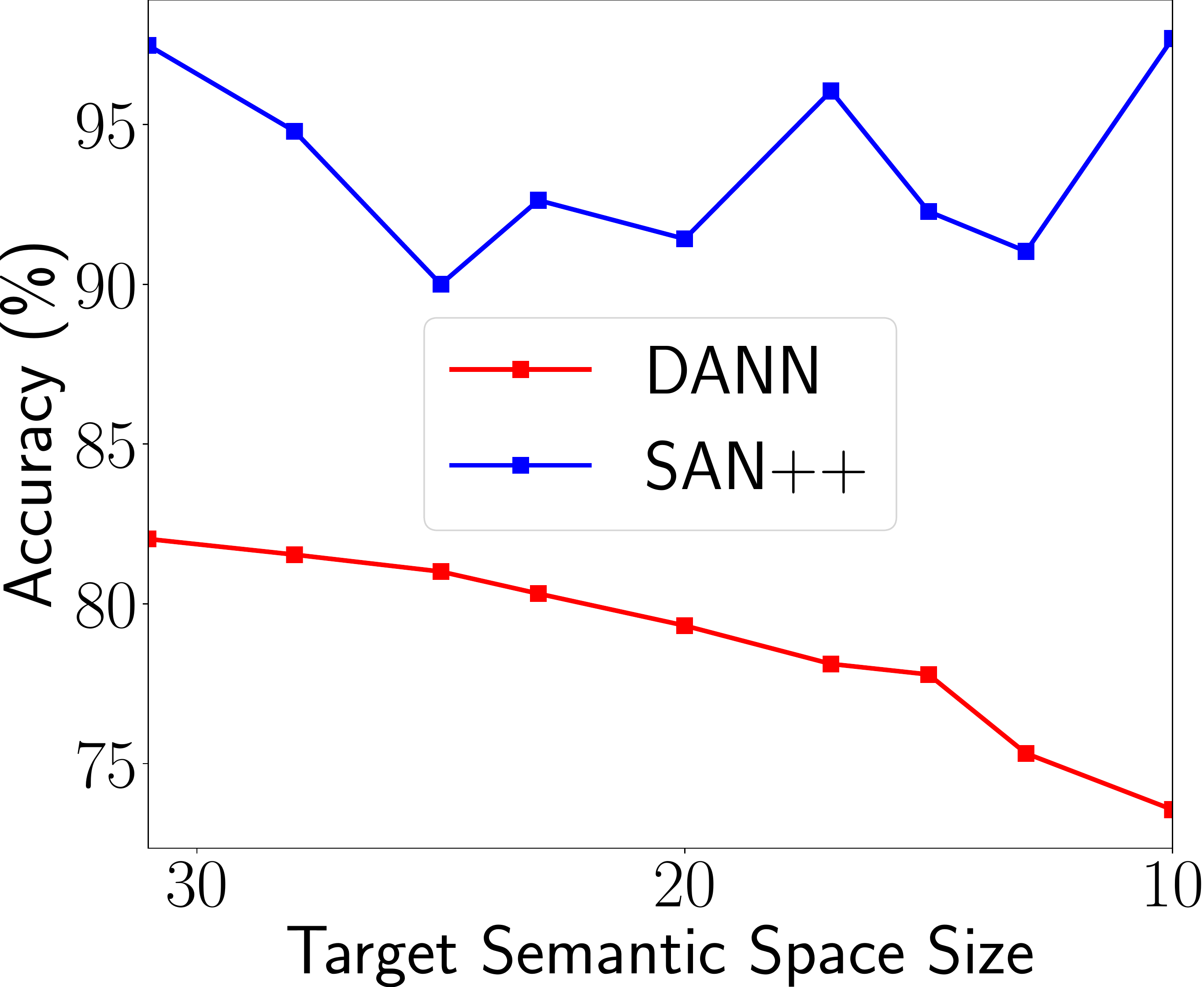}\label{fig:accuracy_number}}
	\hfil
	\subfigure[Transferable Probability]{\includegraphics[width=.23\textwidth]{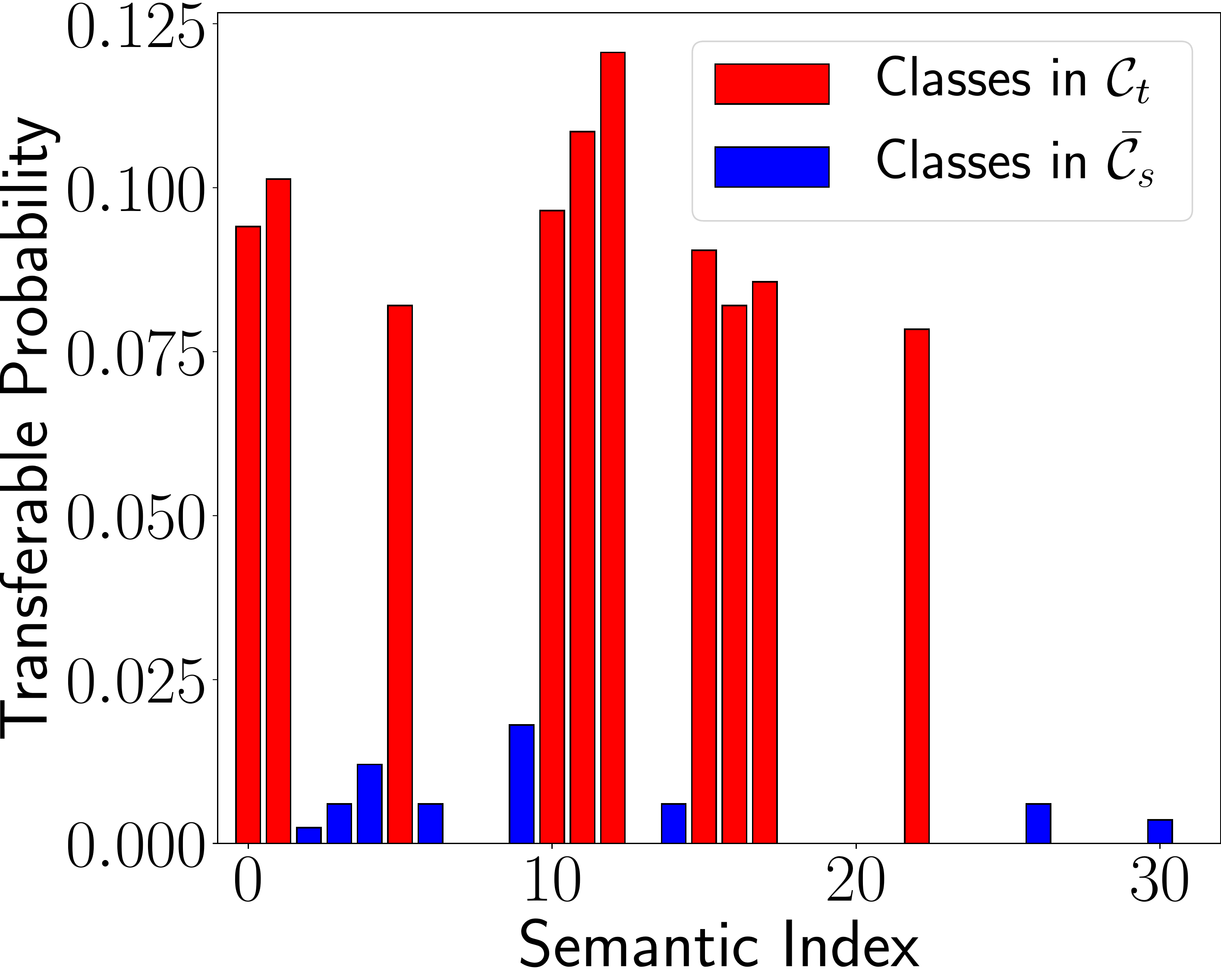}\label{fig:class_weight}}
	\vspace{-10pt}
	\caption{(a) Accuracy by varying the size of target class space on A$\rightarrow$W; (b) The estimated class transferable probability $\mathbf{w}$ for the task A$\rightarrow$W.}
\end{figure}

\subsubsection{Target Class Space}

We investigate a wider spectrum of partial domain adaptation problems by varying the size of target class space on task \textbf{A} $\rightarrow$ \textbf{W}. Fig.~\ref{fig:accuracy_number} shows that when the target class space size decreases, the performance of DANN degrades, meaning that negative transfer becomes severer when the gap between source and target class spaces are enlarged. The performance of SAN++ degenerates when the target class space size decreases from $31$ to $23$, where the negative transfer problem arises but the transfer problem itself is still hard. The performance of SAN++ increases when the target class space size decreases from $23$ to $10$, where the transfer problem itself becomes easier (the estimation of the class transferable probability becomes easier if the target class space is smaller). The margin that SAN++ outperforms DANN becomes larger when the target class space size decreases. SAN++ also outperforms DANN in the standard domain adaptation setting when the number of target class space contains all $31$ categories.

\begin{figure}[ht]
	\centering
	\includegraphics[width=\columnwidth]{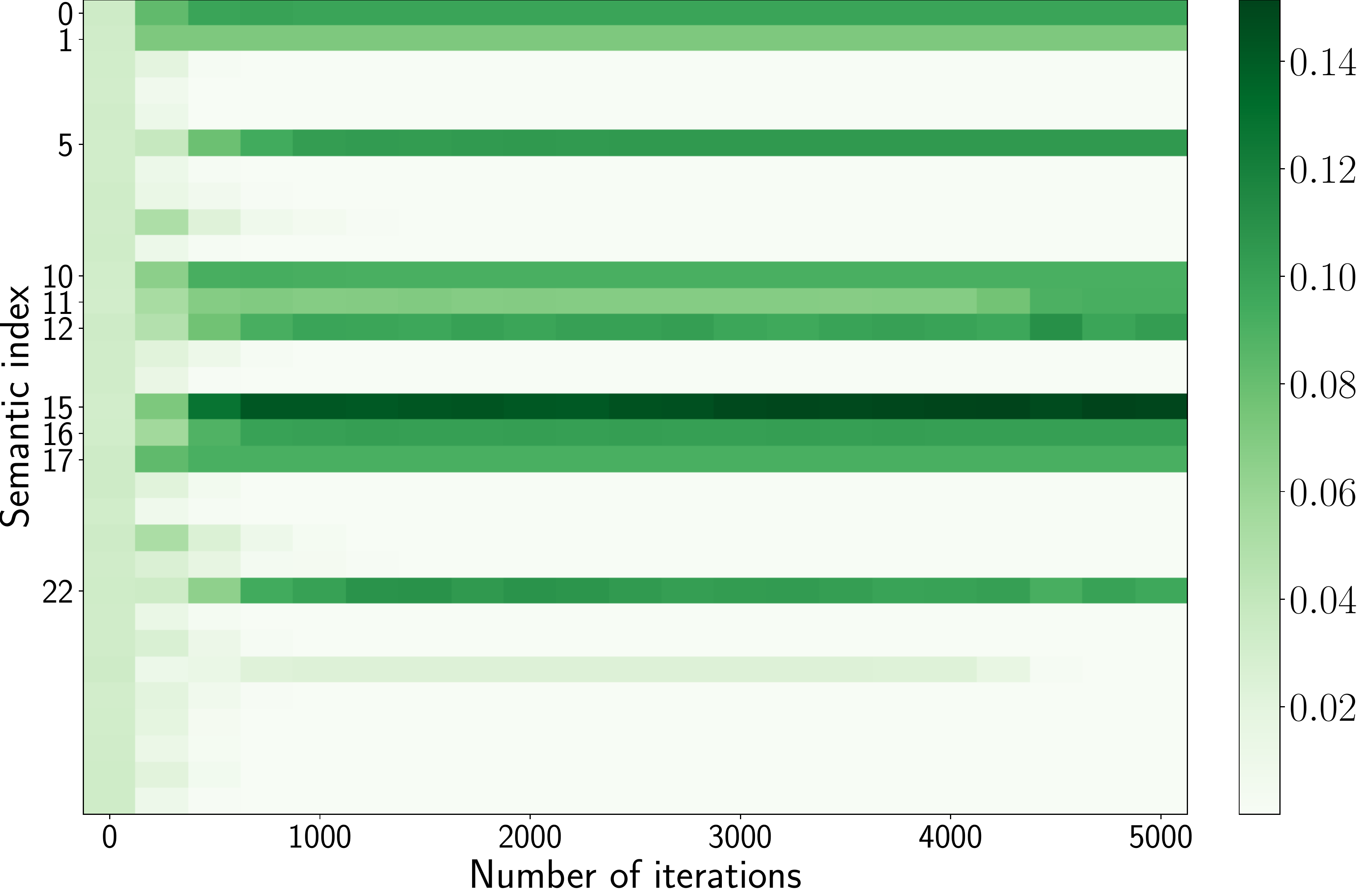}
	\caption{The class transferable probability $\mathbf{w}$ evolves and improves throughout training. The shared classes in $\mathcal{C}$ are upweighted while the outlier source classes in $\bar{\mathcal{C}}_s$ are filtered out progressively.}
	\vspace{-10pt}
	\label{fig:changing_weight}
\end{figure}

\subsubsection{Class Transferable Probability}

We show the class transferable probability $\mathbf{w}$ in Fig.~\ref{fig:class_weight} on the task of A$\rightarrow$W of Office-31. We observe that SAN++ can assign higher transferable probability for the classes in the shared class space $\mathcal{C}$ and filter out outlier classes in $\bar{\mathcal{C}}_s$ with lower transferable probability.

We calculate the class transferable probability $\mathbf{w}$ on task {A} $\rightarrow${W} of Office-31 during training and plot it in Fig~\ref{fig:changing_weight}. In the beginning, the transferable probability is distributed uniformly among all source classes since the classifier is randomly initialized. As the training process proceeds, the transferable probability of the classes in $\mathcal{C}$ increases and those in $\bar{\mathcal{C}}_s$ decreases and finally vanishes to $0$. In the end, all classes in $\mathcal{C}$ are assigned with large transferable probability while classes in $\bar{\mathcal{C}}_s$ are ignored. In particular, the transferable probability of the class ``{ruler}'' (row {25}) in $\bar{\mathcal{C}}_s$ increases moderately at iterations $0\sim4000$, but vanishes finally. This phenomenon may be caused by the similarity between this class and the class ``{bookcase}'' (row {11}) in $\mathcal{C}$ due to the similar rectangle shape. However, our method successfully identifies it and reduces its transferable probability. This shows that SAN++ is robust to small errors in the beginning and is able to make self-corrections through selective adversarial training as shown in Section~\ref{sec:conclusion}.

\begin{figure}[ht]
	\centering
	\includegraphics[width=1.0\columnwidth]{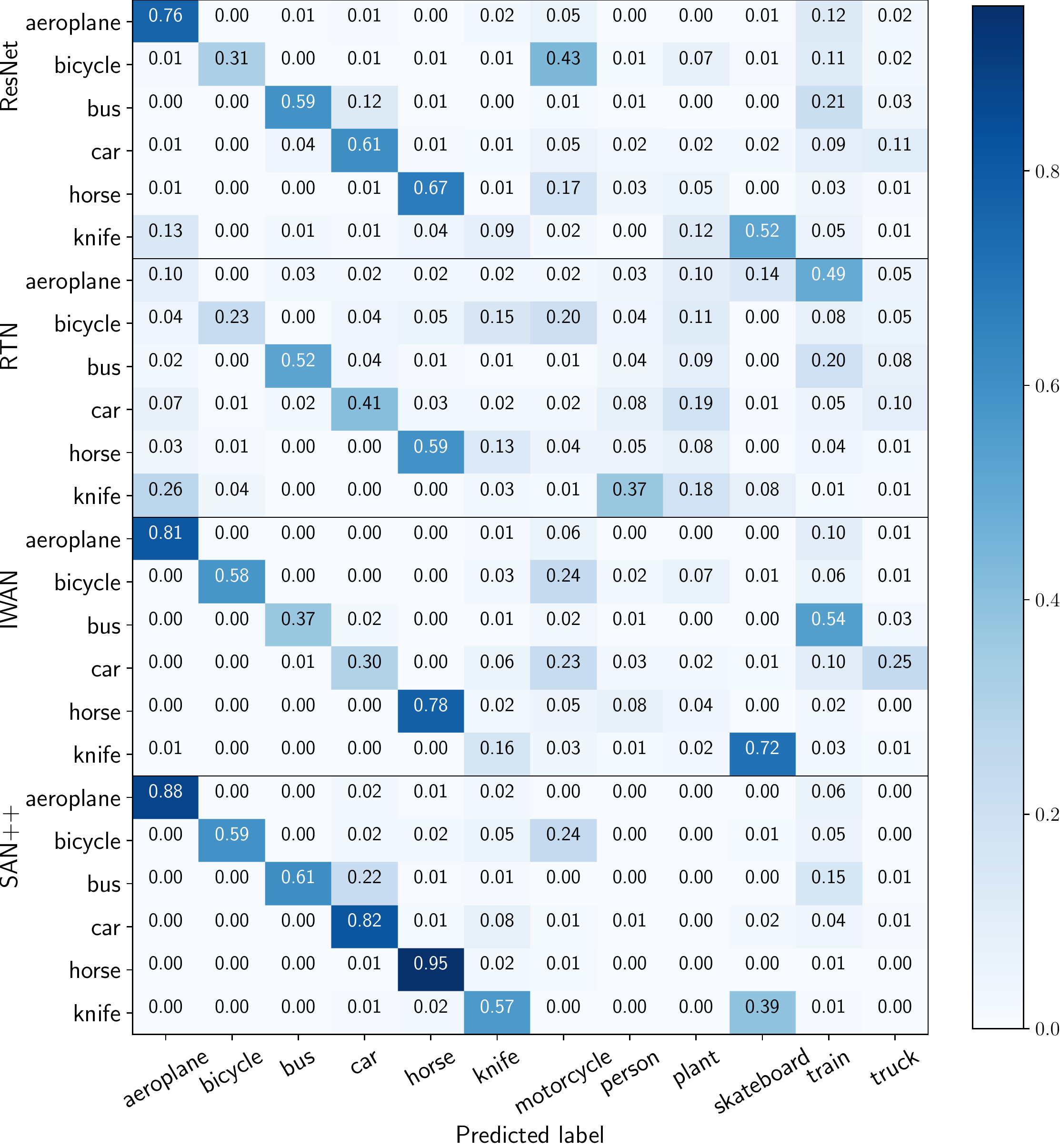}
	\caption{Confusion matrix of prediction results of ResNet, RTN, IWAN and SAN++ on VisDA-2017 dataset.}
	\vspace{-10pt}
	\label{fig:confusion_matrix}
\end{figure}

\begin{table*}[ht]
	\addtolength{\tabcolsep}{-3pt}
	\centering
	\captionsetup{justification=centering}
	\caption{{Comparison with BA$^3$US on \emph{Office-Home} (ResNet-50).}}
	\vspace{-5pt}
	\label{table:accuracy_officehome_1}
	\resizebox{\textwidth}{!}{%
		\begin{tabular}{lccccccccccccc}
			\toprule
			\multirow{2}{30pt}{Method} & \multicolumn{13}{c}{Office-Home} \\
			\cmidrule{2-14}
			                              & {Ar}$\rightarrow${Cl} & {Ar}$\rightarrow${Pr} & {Ar}$\rightarrow${Rw} & {Cl}$\rightarrow${Ar} & {Cl}$\rightarrow${Pr} & {Cl}$\rightarrow${Rw} & {Pr}$\rightarrow${Ar} & {Pr}$\rightarrow${Cl} & {Pr}$\rightarrow${Rw} & {Rw}$\rightarrow${Ar} & {Rw}$\rightarrow${Cl} & {Rw}$\rightarrow${Pr} & Avg            \\
			\midrule
			BA$^3$US~\cite{liang2020baus} & 60.62                 & \textbf{83.16}        & 88.39                 & 71.75                 & 72.79                 & 83.40                 & \textbf{75.45}        & 61.59                 & 86.53                 & 79.25                 & 62.80                 & 86.05                 & 75.98          \\
			SAN++                         & 61.25                 & 81.57                 & 88.57                 & 72.82                 & 76.41                 & 81.94                 & 74.47                 & 57.73                 & 87.24                 & 79.71                 & 63.76                 & 86.05                 & 75.96          \\
			SAN++ w/ COT                  & \textbf{61.49}        & 82.07                 & \textbf{88.74}        & \textbf{73.00}        & \textbf{78.04}        & \textbf{84.59}        & 74.84                 & \textbf{63.40}        & \textbf{87.96}        & \textbf{81.73}        & \textbf{64.72}        & \textbf{87.84}        & \textbf{77.37} \\
			\bottomrule
		\end{tabular}
	}
\end{table*}

\subsubsection{Confusion Matrix}

To investigate the influence of negative transfer, we calculate the confusion matrix on VisDA-2017 task for ResNet, RTN, IWAN and SAN++ (see Fig.~\ref{fig:confusion_matrix}). We observe that RTN suffers from negative transfer in all the classes since it aligns the whole source domain with the target domain, resulting in more predictions of classes into $\bar{\mathcal{C}}_s$. Confusion matrices of SAN++ and IWAN, however, are more``diagonal", which show better performance than ResNet in most classes. This indicates that SAN++ and IWAN introduce less negative transfer and assign less predictions into $\bar{\mathcal{C}}_s$. Fig.~\ref{fig:confusion_matrix} also shows that SAN++ is superior to IWAN. Note that VisDA2017 dataset is not balanced w.r.t. the classes. The class ``car'' has almost twice the number of images compared with other classes. SAN++ performs well on the ``car'' class and thus shows promising results in Table~\ref{table:accuracy_visic}. The accuracy of quite a few classes is still not satisfying. But this is mainly due to some similar categories like bicycle/motorcycle, bus/car/train, knife/skateboard (both are elongated). This is a widely-known open problem in unsupervised domain adaptation and is not the focus of this paper.

\begin{table}[htbp]
	\addtolength{\tabcolsep}{3pt}
	\centering
	\caption{{Comparison with BA$^3$US on \emph{VisDA-2017}, \emph{ImageNet-Caltech} and \emph{OpenMIC} (ResNet-50).}}
	\label{table:accuracy_visic_1}
	\vspace{-5pt}
	\resizebox{1\columnwidth}{!}{
		\begin{tabular}{lcccc}
			\toprule
			\multirow{2}{30pt}{Method} & \multirow{2}{30pt}{\centering VisDA2017} & \multicolumn{2}{c}{ImageNet-Caltech} & \multirow{2}{30pt}{Open MIC}\\
			\cmidrule(lr){3-4}
			                              &                & I $\rightarrow$ C & C $\rightarrow$ I &                \\
			\midrule
			BA$^3$US~\cite{liang2020baus} & 46.39          & 84.00             & 83.35             & 76.95          \\
			SAN++                         & 63.06          & 83.34             & 81.07             & 77.56          \\
			SAN++ w/ COT                  & \textbf{76.94} & \textbf{85.15}    & \textbf{84.76}    & \textbf{78.88} \\
			\bottomrule
		\end{tabular}}
\end{table}

\begin{table}[htbp]
	\addtolength{\tabcolsep}{-4.1pt}
	\centering
	\caption{{Comparison with BA$^3$US on \emph{Office-31} (ResNet-50).}}
	\label{table:accuracy_office_1}
	\vspace{-5pt}
	\resizebox{1\columnwidth}{!}{
		\begin{tabular}{lccccccc}
			\toprule
			\multirow{2}{30pt}{Method} & \multicolumn{7}{c}{Office-31} \\
			\cmidrule{2-8}
			                              & A $\rightarrow$ W & D $\rightarrow$ W & W $\rightarrow$ D & A $\rightarrow$ D & D $\rightarrow$ A & W $\rightarrow$ A & Avg            \\
			\midrule
			BA$^3$US~\cite{liang2020baus} & 98.98             & \textbf{100.00}   & 98.73             & \textbf{99.36}    & 94.82             & 94.99             & 97.81          \\
			SAN++                         & \textbf{99.66}    & \textbf{100.00}   & \textbf{100.00}   & 98.09             & 94.05             & 95.51             & 97.89          \\
			SAN++ w/ COT                  & 99.32             & \textbf{100.00}   & \textbf{100.00}   & 98.72             & \textbf{94.95}    & \textbf{96.14}    & \textbf{98.18} \\
			\bottomrule
		\end{tabular}}
\end{table}

\subsubsection{Complement Entropy}

{We compare with the state-of-the-art distribution alignment based partial domain adaptation method BA$^3$US~\cite{liang2020baus}. BA$^3$US adopts the complement entropy (COT) loss~\cite{chen2019complement} to regularize the prediction, which is a general technique to expect uniform and low prediction scores for incorrect classes for labeled source samples. COT can be naturally embedded into our method and we also show SAN++ with COT. As shown in Tables~\ref{table:accuracy_officehome_1}, ~\ref{table:accuracy_office_1} and~\ref{table:accuracy_visic_1}, SAN++ performs comparably to BA$^3$US on most of the tasks in Office-31, Office-Home, ImageNet-Caltech and Open MIC datasets, but on the VisDA2017 tasks SAN++ significantly outperforms BA$^3$US. We further show that when combining SAN++ with COT, the performance can be improved to be higher than both SAN++ and BA$^3$US. The results demonstrate that SAN++ is complementary to COT~\cite{chen2019complement}, and the two methods can be combined to produce much higher performance in practical applications.}

\begin{table}[htbp]
	\addtolength{\tabcolsep}{2pt}
	\centering
	\caption{Memory and Time usage of SAN++ and SAN++ w/o bottom layer sharing (one iteration per image on I$\rightarrow$C).}
	\label{table:memory_time}
	\vspace{-5pt}
	\resizebox{\columnwidth}{!}{
		\begin{tabular}{lccc}
			\toprule
			Method           & Peak GPU Memory/MB & Time/s \\
			\midrule
			SAN++            & 569.5              & 0.148  \\
			SAN++ w/o shared & 5551.0             & 1.601  \\
			\bottomrule
		\end{tabular}}
\end{table}

\begin{table*}[htbp]
	\addtolength{\tabcolsep}{-1pt}
	\centering
	\caption{Accuracy (\%) of SAN++ and its variants on \emph{Office-31} with ResNet-50 backbone.}
	\label{table:accuracy_ablation}
	\vspace{-5pt}
	\resizebox{1\textwidth}{!}{
		\begin{tabular}{ccccc|ccccccc}
			\toprule
			\multirow{2}{30pt}{\centering instance} & \multirow{2}{30pt}{\centering class} & \multirow{2}{50pt}{\centering self-training} & \multirow{2}{30pt}{\centering entropy} & \multirow{2}{30pt}{\centering shared}  & \multicolumn{7}{|c}{Office-31} \\
			\cmidrule{6-12}
			       &        &        &        &        & A $\rightarrow$ W & D $\rightarrow$ W & W $\rightarrow$ D & A $\rightarrow$ D & D $\rightarrow$ A & W $\rightarrow$ A & Avg            \\
			\midrule
			\xmark & \xmark & \xmark & \xmark & \xmark & 75.59             & 96.27             & 98.09             & 83.44             & 83.92             & 84.97             & 87.05          \\
			\ymark & \xmark & \xmark & \xmark & \ymark & 85.47             & 98.22             & 99.07             & 87.24             & 87.89             & 88.43             & 91.05          \\
			\ymark & \ymark & \xmark & \xmark & \ymark & 90.11             & 99.19             & 99.59             & 89.01             & 90.69             & 90.87             & 93.24          \\
			\ymark & \ymark & \xmark & \ymark & \ymark & 95.01             & 99.69             & 99.69             & 91.77             & 92.66             & 93.10             & 95.32          \\
			\ymark & \ymark & \ymark & \xmark & \xmark & 97.89             & 99.88             & 99.69             & 95.03             & 93.12             & \textbf{96.48}    & 97.02          \\
			\midrule
			\ymark & \ymark & \ymark & \xmark & \ymark & \textbf{99.66}    & \textbf{100.00}   & \textbf{100.00}   & \textbf{98.09}    & \textbf{94.05}    & {95.51}           & \textbf{97.89} \\
			\bottomrule
		\end{tabular}}
\end{table*}

\subsubsection{Model Ablations}

In terms of whether to use shared bottom layers for the multitask discriminator, we can observe that using shared bottom layers performs comparably to using entirely separate discriminators. As shown in Table~\ref{table:memory_time}, using shared layers can decrease the memory and computational cost significantly.

To go deeper with the efficacy of all SAN++ components: the instance selection (\emph{instance}), the class selection (\emph{class}), the self-training regularization (\emph{self-training}), the entropy minimization used in our conference version but replaced by self-training in this journal version (\emph{entropy}), and whether to share bottom layers for the multi-task discriminator (\emph{sharing}). We perform the ablation study by using or removing these components.

Based on ablation study in Table~\ref{table:accuracy_ablation}, we can make the following observations. (1) Self-training is a more effective regularization technique for partial domain adaptation, and can be used to replace the widely-used entropy minimization. (2) Self-training enhanced by the proposed class transferable probability can substantially improve the performance of partial domain adaptation. (3) Class selection can improve the performance by confining the training within the shared class space to mitigate negative transfer. (4) Instance selection enables a more powerful distribution alignment in a class-correspondence way to promote positive transfer.

\section{Towards More Challenging PDA}

{This paper mainly studies partial domain adaptation where the source class space subsumes the target class space. This PDA setting is straightforward and easy to understand. However, due to practical issues like synonyms or fine-grain species in class labels, it is difficult to find a source domain whose class space \textit{exactly} contains the target class space.}

{For example, the source domain may have a ``Dog'' class, which consists of samples from fine-grained species like ``Collie'', ``Golden Retriever'', etc. All these samples are labeled as ``Dog'' since fine-grained labels are unavailable. Typically, the target domain is smaller than the source domain and requires fine-grained labels like ``Golden Retriever''. The big source domain contains samples of ``Golden Retriever'' among samples of ``Dog'', which should help discriminating ``Golden Retriever'' in the target domain. }

{This scenario seems to violate the PDA setting since ``Golden Retriever'' $\in \mathcal{C}_t$ and ``Golden Retriever'' $\notin \mathcal{C}_s \implies \mathcal{C}_t \nsubseteq \mathcal{C}_s$. However, ``Dog'' is a superclass of ``Golden Retriever'' and ``Dog'' $\in \mathcal{C}_s \cap \mathcal{C}_t$. From the perspective of \emph{superclass}, this scenario can be regarded as PDA because source classes subsume superclasses of target classes, \emph{i.e.}, the class space inclusion assumption $\mathcal{C}_t \subseteq \mathcal{C}_s$ is satisfied at the level of superclass. }

{In manually curated datasets, the assumption that the source label space $\mathcal{C}_s$ exactly subsumes the target label space $\mathcal{C}_t$ can easily hold. In practice, however, people may find that the label space inclusion assumption cannot hold exactly, but it may hold at the level of superclass like the above example. This scenario is more challenging, since PDA learning cannot directly yield a classifier to organize target data into fine-grained target labels (such as ``Golden Retriever''), but it can only produce a transferable feature extractor $F$, which can be useful for learning a target classifier from a few examples indicative of the target classes.}

\begin{table}[htbp]
	\centering
	\renewcommand\tabcolsep{10pt}
	\caption{{Superclasses in the VisDA-Sub setting.}}
	\label{tab:visda_sub_detail}
	\vspace{-5pt}
	\resizebox{0.8\columnwidth}{!}{
		\begin{tabular}{c|c}
			\toprule
			superclass             & subclass            \\
			\\[-5pt]
			\hline
			\\[-5pt]
			cycling transportation & bicycle, motorcycle \\
			\\[-5pt]
			\hline
			\\[-5pt]
			wheel transportation   & bus, car, truck     \\
			\\[-5pt]
			\hline
			\\[-5pt]
			others                 & knife, skateboard   \\
			\\[-5pt]
			\hline
			\\[-5pt]
			other transportation   & aeroplane, train    \\
			\\[-5pt]
			\hline
			\\[-5pt]
			organism               & horse, plant        \\
			\\[-5pt]
			\hline
			\\[-5pt]
			person                 & person              \\
			\bottomrule
		\end{tabular}}
\end{table}

{To explore the challenging PDA setting, we construct new classes as superclasses (Table~\ref{tab:visda_sub_detail}) and use the original classes in VisDA-2017 as subclasses. The setting is called ``VisDA-Sub'', with 12 tasks having different subclasses in the target domain, where 8 tasks with prefix ``Sub'' have subclasses from different superclasses in the target domain while 4 tasks with prefix ``Sub-Same'' have subclasses from the same superclass in the target domain. The details of the target domain subclasses are shown in Table~\ref{tab:visda_sub}.}

{For the VisDA-Sub setting, since the target class space is unknown during training, for SAN++ and all the compared methods, we can at most learn a transferable feature extractor. To enable fine-grained recognition (subclasses) in the target domain, we randomly select \emph{only one} labeled example from each class in the target class space as the prototype. We predict the label of each target example as the label of the nearest prototype. We repeat the experiment $10$ times with different target prototypes and report the average accuracy.}

{Results on the more challenging PDA tasks are shown in Table~\ref{table:accuracy_visda_sub}. Note that whether the target classes belong to different superclasses or belong to the same superclasses, SAN++ outperforms all the baseline methods. }

\begin{table*}[tbp]
	\centering
	\renewcommand\tabcolsep{7pt}
	\caption{{The details of the target classes and the corresponding superclasses for each task in VisDA-Sub.}}
	\vspace{-5pt}
	\label{tab:visda_sub}
	\resizebox{1\textwidth}{!}{
		\begin{tabular}{c|cc|ccc|cc|cc|cc|c}
			\toprule
			superclass & \multicolumn{2}{|c|}{cycling transportation} & \multicolumn{3}{c|}{wheel transportation} & \multicolumn{2}{c|}{others} & \multicolumn{2}{c|}{other transportation} & \multicolumn{2}{c|}{organism} & Person \\
			\cmidrule{1-13}
			subclass            & bicycle & motorcycle & bus    & car    & truck  & knife  & skateboard & aeroplane & train  & horse & plant  & person \\
			\midrule
			\textbf{Sub-2-1}    & \ymark  &            & \ymark &        &        &        &            &           &        &       &        &        \\
			\midrule
			\textbf{Sub-2-2}    &         & \ymark     &        &        & \ymark &        &            &           &        &       &        &        \\
			\midrule
			\textbf{Sub-3-1} & \ymark &  & \ymark &  &  & \ymark &  &  &  &  & \\
			\midrule
			\textbf{Sub-3-2}    &         & \ymark     &        &        & \ymark &        & \ymark     &           &        &       &        &        \\
			\midrule
			\textbf{Sub-4-1} & \ymark &  & \ymark &  &  & \ymark &  & \ymark &  &  & \\
			\midrule
			\textbf{Sub-4-2}    &         & \ymark     &        &        & \ymark &        & \ymark     &           & \ymark &       &        &        \\
			\midrule
			\textbf{Sub-5-1} & \ymark &  & \ymark &  &  & \ymark &  & \ymark &  & \ymark & \\
			\midrule
			\textbf{Sub-5-2}    &         & \ymark     &        &        & \ymark &        & \ymark     &           & \ymark &       & \ymark &        \\
			\midrule
			\textbf{Sub-Same-1} & \ymark  & \ymark     &        &        &        &        &            &           &        &       &        &        \\
			\midrule
			\textbf{Sub-Same-2} & \ymark  & \ymark     & \ymark & \ymark & \ymark &        &            &           &        &       &        &        \\
			\midrule
			\textbf{Sub-Same-3} & \ymark  & \ymark     & \ymark & \ymark & \ymark & \ymark & \ymark     &           &        &       &        &        \\
			\midrule
			\textbf{Sub-Same-4} & \ymark  & \ymark     & \ymark & \ymark & \ymark & \ymark & \ymark     & \ymark    & \ymark &       &        &        \\
			\bottomrule
		\end{tabular}}
\end{table*}

\begin{table*}[hptb]
	\addtolength{\tabcolsep}{-3pt}
	\centering
	\caption{{Accuracy (\%) of partial domain adaptation tasks on \emph{VisDA-Sub} (ResNet-50).}}
	\vspace{-5pt}
	\label{table:accuracy_visda_sub}
	\resizebox{1\textwidth}{!}{
		\begin{tabular}{lccccccccccccc}
			\toprule
			\multirow{2}{30pt}{Method} & \multicolumn{13}{c}{VisDA-Sub} \\
			\cmidrule{2-14}
			                              & Sub-2-1        & Sub-2-2        & Sub-3-1        & Sub-3-2        & Sub-4-1        & Sub-4-2        & Sub-5-1        & Sub-5-2        & Sub-Same-1     & Sub-Same-2     & Sub-Same-3     & Sub-Same-4     & Avg            \\
			\midrule
			ResNet~\cite{cite:CVPR16DRL}  & 71.74          & 73.21          & 70.61          & 66.84          & 75.34          & 74.25          & 76.20          & 74.55          & 55.74          & 37.63          & 38.67          & 38.29          & 62.76          \\
			DANN~\cite{cite:ICML15DANN}   & 71.79          & 68.88          & 59.92          & 51.32          & 68.71          & 66.52          & 71.51          & 73.90          & 57.02          & 31.67          & 26.28          & 32.88          & 56.45          \\
			ADDA~\cite{cite:CVPR17ADDA}   & 72.45          & 67.32          & 61.56          & 53.21          & 61.87          & 64.13          & 71.21          & 70.80          & 55.33          & 34.15          & 28.19          & 29.99          & 55.60          \\
			IWAN~\cite{cite:CVPR18IWAN}   & 77.49          & 69.54          & 62.12          & 63.12          & 71.23          & 65.66          & 70.44          & 73.11          & 57.55          & 38.31          & 31.45          & 35.02          & 59.59          \\
			PADA~\cite{cao_partial_2018}  & 78.56          & 70.39          & 63.12          & 64.35          & 70.84          & 66.19          & 72.32          & 75.49          & 58.38          & 39.12          & 32.01          & 35.11          & 60.49          \\
			DRCN~\cite{li2020deep}        & 79.15          & 71.43          & 63.31          & 63.41          & 71.35          & 70.41          & 73.12          & 74.88          & 58.76          & 40.46          & 34.32          & 36.12          & 61.39          \\
			ETN~\cite{cao_learning_2019}  & 80.34          & 72.13          & 66.67          & 66.18          & 75.97          & 76.49          & 75.12          & 76.39          & 59.46          & 40.94          & 40.58          & 37.21          & 63.96          \\
			BA$^3$US~\cite{liang2020baus} & 85.84          & 73.53          & 69.61          & 56.50          & 77.24          & 67.71          & \textbf{78.52} & 77.39          & 60.01          & 42.09          & 41.96          & 38.49          & 64.07          \\
			\midrule
			SAN~\cite{cite:CVPR18SAN}     & 77.27          & 71.56          & 62.46          & 64.11          & 73.67          & 78.14          & 73.84          & 81.56          & 57.88          & 38.59          & 31.02          & 34.96          & 62.09          \\
			\midrule
			SAN++                         & \textbf{88.51} & \textbf{80.60} & \textbf{79.72} & \textbf{76.71} & \textbf{77.28} & \textbf{86.15} & 78.27          & \textbf{89.07} & \textbf{62.87} & \textbf{45.88} & \textbf{43.44} & \textbf{41.62} & \textbf{71.84} \\
			\bottomrule
		\end{tabular}}
\end{table*}

{When different target classes belong to different superclasses (``Sub-2-1''/``Sub-2-2''/``Sub-3-1''/``Sub-3-2''/``Sub-4-1''/``Sub-4-2''/``Sub-5-1''/``Sub-5-2''), it is similar to the common PDA setting since no subclasse belong to the same superclass and the class boundary of the superclasses can also be used for separating the subclasses. However, the semantic shift between the superclass and the subclass still makes the task more difficult than common PDA tasks. In these tasks, all the PDA methods outperform ResNet and SAN++ achieves the best performance.}

{When there are target classes belonging to the same superclasses (``Sub-Same-1''/``Sub-Same-2''/``Sub-Same-3''/``Sub-Same-4''), it is much more difficult than common PDA tasks due to the need to separate subclasses in the same superclass. Even some advanced PDA methods such as PADA and DRCN perform worse than ResNet (source only). This is because PDA methods can only discriminate superclasses, ignoring more detailed separation of target subclasses. SAN++ could learn a feature space where samples of different subclasses are more clearly separated. }

\section{Conclusion}

In this paper, we introduce the partial domain adaptation (PDA) problem, where the source class space subsumes the target class space. We theoretically analyze the problem and show the importance of estimating the transferable probability of each class and each instance across domains. Based on the theoretical guidelines, we propose the Selective Adversarial Networks (SAN and SAN++) to promote positive transfer in the shared class space by instance selection, and to alleviate negative transfer of outlier source class space by class selection. Both theoretical and empirical results demonstrate that our approach successfully tackles the partial domain adaptation problem, even when the PDA task becomes more challenging with superclasses.

\section*{Acknowledgments}
This work was supported by National Key R\&D Program of China (2020AAA0109201), NSFC grants (62022050, 62021002), Beijing Nova Program (Z201100006820041), BNRist Fund (BNR2021RC01002), and Tsinghua-Huawei Innovation Fund.

{
\small
\bibliographystyle{plain}
\bibliography{IEEEabrv,SAN}
}

\vspace{-20pt}

\begin{IEEEbiography}[{\includegraphics[width=1in,height=1.25in,clip,keepaspectratio]{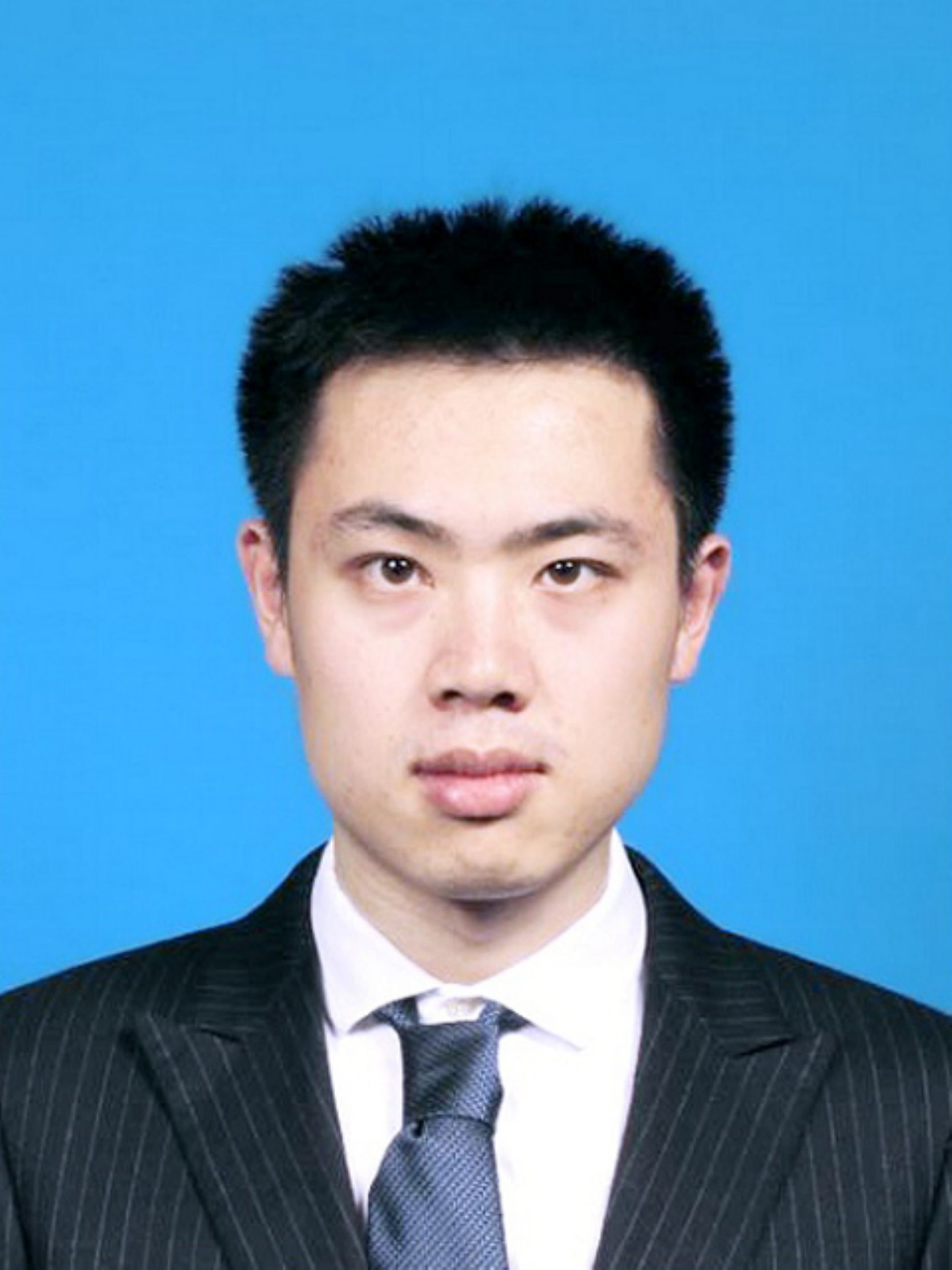}}]{Zhangjie Cao}
	received the B.E. degree in Computer Software from Tsinghua University, China in 2018. He is pursuing the Ph.D. degree in Department of Computer Science at Stanford University. His research interests include machine learning and computer vision.
\end{IEEEbiography}

\vspace{-20pt}

\begin{IEEEbiography}[{\includegraphics[width=1in,height=1.25in,clip,keepaspectratio]{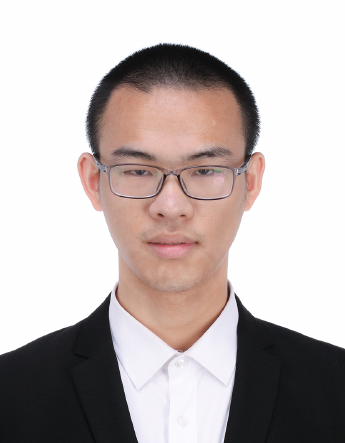}}]{Kaochao You}
	received the B.E. degree in Software Engineering from Tsinghua University, China in 2020. He is pursuing the Ph.D. degree in the School of Software, Tsinghua University. His research interests include machine learning and computer vision.
\end{IEEEbiography}

\vspace{-20pt}

\begin{IEEEbiography}[{\includegraphics[width=1in,height=1.25in,clip,keepaspectratio]{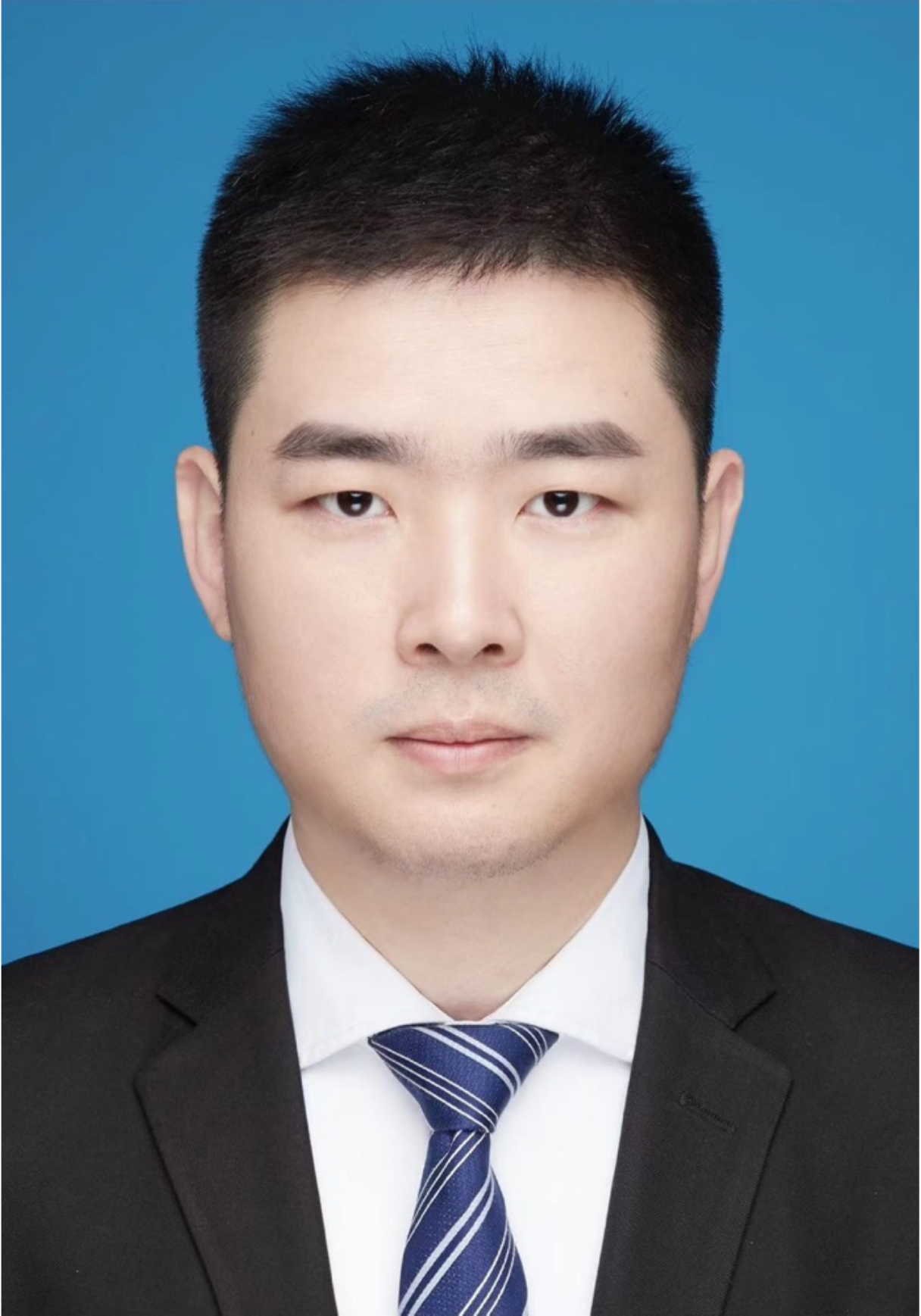}}]{Ziyang Zhang}
	received the Ph.D. degree in instrument science and technology from Tsinghua University in 2019. He is a senior engineer in 2012 lab of Huawei Technology Co., Ltd. His research interests include neuromorphic computing, event camera, digital human, automatic driving and brain computer interface.
\end{IEEEbiography}

\vspace{-20pt}

\begin{IEEEbiography}[{\includegraphics[width=1in,height=1.25in,clip,keepaspectratio]{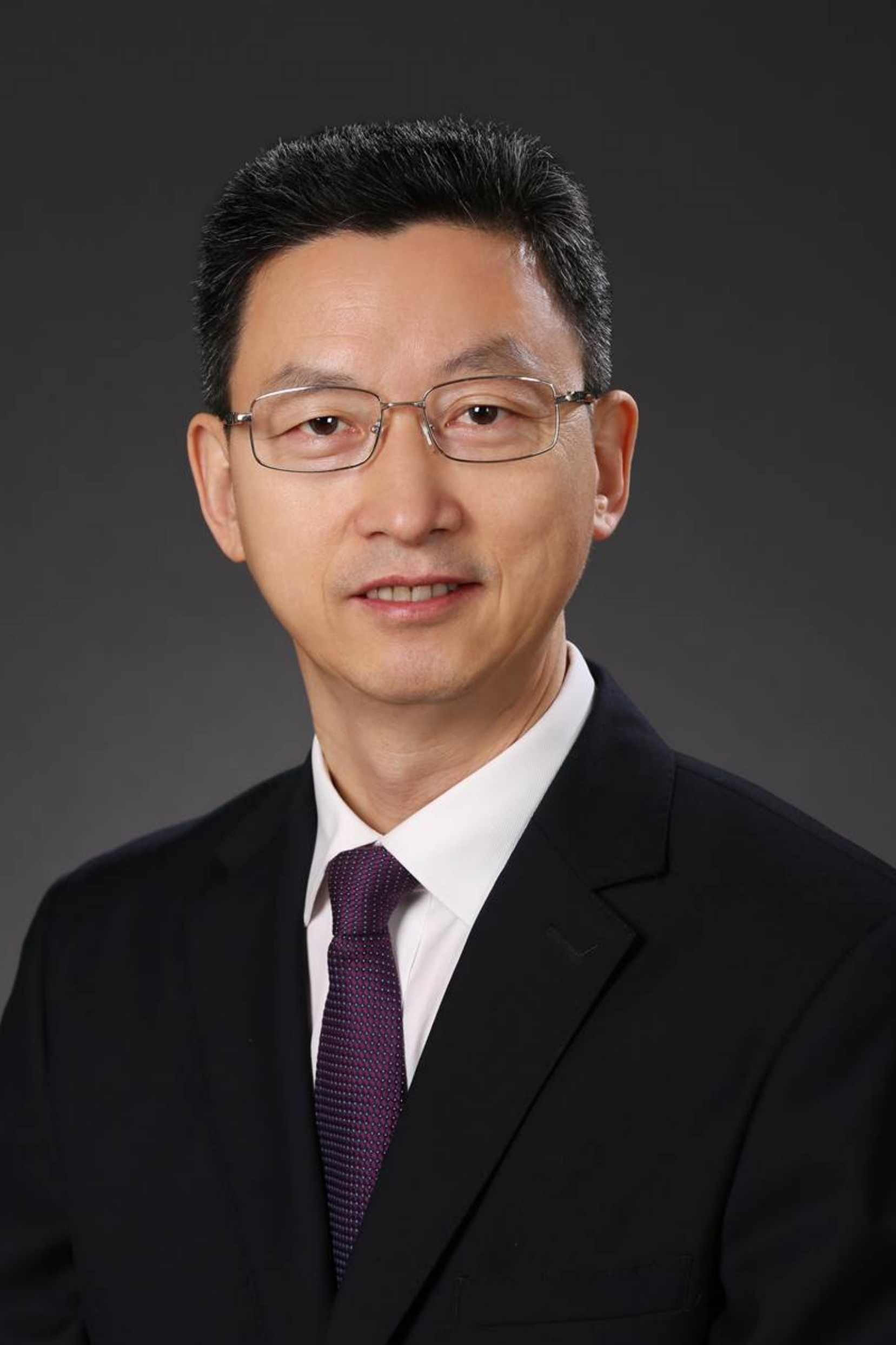}}]{Jianmin Wang}
	graduated from Peking University, China in 1990, and received the M.E. and Ph.D. degrees in Computer Software from Tsinghua University, China in 1992 and 1995, respectively. He is a full professor in the School of Software, Tsinghua University. His research interests include big data management systems, workflow and BPM technology, and large-scale data analytics. He led to develop a product data \& lifecycle management system, which has been deployed in hundreds of enterprises in China. He is leading to develop a big data system software stack named Tsinghua DataWay.
\end{IEEEbiography}

\vspace{-20pt}

\begin{IEEEbiography}[{\includegraphics[width=1in,height=1.25in,clip,keepaspectratio]{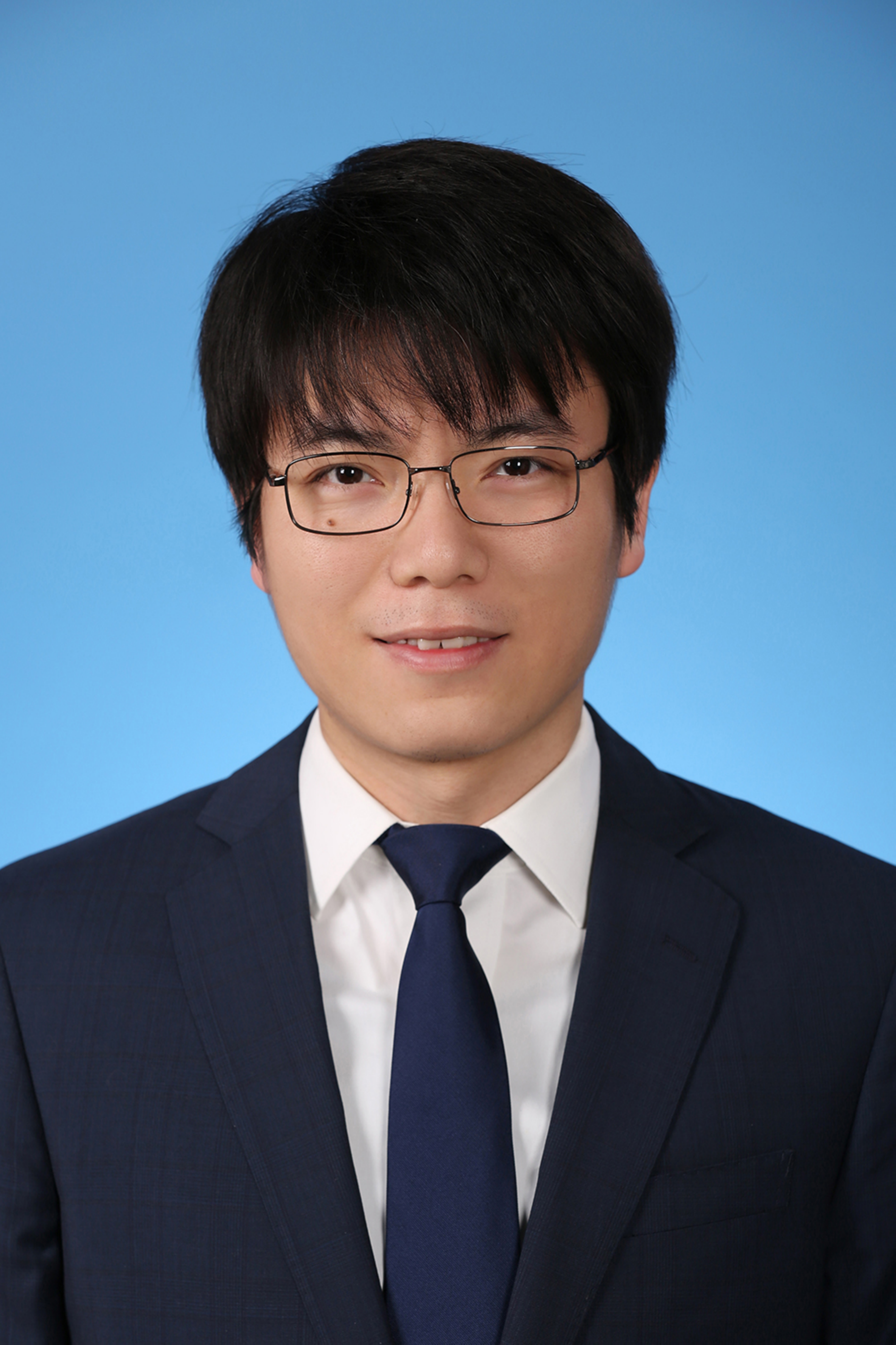}}]{Mingsheng Long}
	received the BE degree in electrical engineering and the PhD degree in computer science from Tsinghua University in 2008 and 2014 respectively. He was a visiting researcher in computer science, UC Berkeley from 2014 to 2015. He is an associate professor with the School of Software, Tsinghua University.	He is an associate editor of the \emph{IEEE Transactions on Pattern Analysis and Machine Intelligence}, and area chairs of major machine learning conferences, including ICML, NeurIPS, and ICLR. His research is dedicated to machine learning theories and algorithms, with special interests in transfer, adaptation, and data-efficient learning, foundation models, learning with spatiotemporal and scientific knowledge.
\end{IEEEbiography}

\end{document}